\pdfoutput=1
\documentclass[a4paper,11pt]{article}
\usepackage{authblk}

\usepackage{hyperref}
\hypersetup{
  colorlinks=true,
  linkcolor=green!30!black,
  linktoc=all,
  citecolor=blue,
  bookmarksnumbered=true,
}

\RequirePackage{etoolbox}
\RequirePackage{amsmath, amsthm}
\RequirePackage{amsfonts}

\RequirePackage{mathtools}

\newcommand{\st}[0]{\ensuremath{\,\mathbf{:}\,}}

\DeclarePairedDelimiter\abs{|}{|}
\newcommand{\norm}[2][]{\ensuremath{\Vert #2 \Vert_{#1}}}

\newcommand{\inb}[1]{\left\{#1\right\}}
\newcommand{\inp}[1]{\left(#1\right)}
\newcommand{\insq}[1]{\left[#1\right]}

\makeatletter
\newcommand*{\defeq}{\mathrel{\rlap{\raisebox{0.3ex}{$\m@th\cdot$}}\raisebox{-0.3ex}{$\m@th\cdot$}}=}
\newcommand*{\eqdef}{=
  \mathrel{\rlap{\raisebox{0.3ex}{$\m@th\cdot$}}\raisebox{-0.3ex}{$\m@th\cdot$}}}
\makeatother

\newcommand{\R}[0]{\ensuremath{\mathbb{R}}}

\DeclarePairedDelimiterX\braket[2]{\langle}{\rangle}{#1\,\delimsize\vert\,\mathopen{}#2}

\renewcommand{\vec}[1]{\mathbf{#1}}

\RequirePackage{cleveref}

\let\oldnewtheorem\newtheorem
\RenewDocumentCommand{\newtheorem}{m o m
  o r<>}{\IfValueTF{#2}{\oldnewtheorem{#1}[#2]{#3} \AddToHook{env/#1/begin}{\crefalias{#2}{#1}} }{\IfValueTF{#4}{\oldnewtheorem{#1}{#3}[#4] }{\oldnewtheorem{#1}{#3}}} \crefname{#1}{#3}{#5} }

\newtheorem{theorem}{Theorem}<Theorems>
<Lemmas>
\newtheorem{observation}[theorem]{Observation}<Observations>
<Claims>
<Conditions>
\newtheorem{example}[theorem]{Example}<Examples>
<Facts>
\newtheorem{corollary}[theorem]{Corollary}<Corollaries>
\newtheorem{proposition}[theorem]{Proposition}<Propositions>
\theoremstyle{definition}
<Definitions>
\newtheorem{remark}[theorem]{Remark}<Remarks>
\newtheorem{problem}[theorem]{Problem}<Problems>
<Exercises>

\crefname{section}{Section}{Sections}
\crefname{appendix}{Appendix}{Appendices}
\crefname{table}{Table}{Tables}

\newtoggle{draft}
\toggletrue{draft}
\iftoggle{draft}{\RequirePackage{xcolor} \RequirePackage{todonotes} 
   \RequirePackage[draft]{fixme} }{ }

\RequirePackage{tikz}
\usetikzlibrary{arrows,automata}

\newcommand{\act}{\textit{Act}}

\newcommand{\initunfair}{\textsc{Low-Entropy}}
\newcommand{\initunfairDec}{\textsc{Low-Entropy-Decision}}
\newcommand{\apxinitunfairDec}[1]{\textsc{Apx-Low-Entropy-Decision}}
\newcommand{\apxinitunfairSynt}[1]{\textsc{Apx-Low-Entropy-Synt}}
\newcommand{\yes}{\emph{YES}} \newcommand{\no}{\emph{NO}}

\newcommand{\memoryless}{\ensuremath{\Pi_1}} 

\newcommand{\dualc}{\ensuremath{\Psi}}

\newcommand{\conc}{\ensuremath{\Gamma}}

\newcommand{\D}{\ensuremath{\mathcal{D}}
}

\newcommand{\predicate}{\ensuremath{\mathcal{C}}}

\newcommand{\diag}{\ensuremath{\mathrm{diag}}}
 
\usepackage{fullpage}

\NewDocumentCommand\citet{o m}{\IfValueTF{#1}{\cite[#1]{#2}}{\cite{#2}}}

\usepackage[T1]{fontenc}
\RequirePackage[ttscale=0.81]{libertine}
\RequirePackage[libertine]{newtxmath}

\newcommand{\arxiv}[1]{\href{https://arxiv.org/abs/#1}{arXiv:#1}}

\begin{document}

\title{Entropy Objectives in Markov Decision Processes}

\author[1]{S. Akshay}
\author[1]{Raghav Goyal}
\author[1]{Aditya Neeraje}
\author[2]{Piyush Srivastava}
\affil[1]{Indian Institute of Technology Bombay, Mumbai, India.
}
\affil[2]{Tata Institute of Fundamental Research, Mumbai, India.}
\date{}
\maketitle

\newcommand{\runningexamplefig}{
\begin{tikzpicture}[
    >=stealth,
    node distance=1.5cm and 1.5cm,
]
\node[state] (A) {$A$};

\node[state] (B) [right=of A] {\(B\)};
\node[state] (C) [right=of B] {\(C\)};

\path[->] (A) edge  node[above] {$a_2,1$} (B) edge [loop above] node [above] {\(a_1, 1\)} () (B) edge node [above] {\(1\)} (C) (C) edge [bend left=30] node [below] {\(1/2\)} (A) edge [loop right] node [right] {\(1/2\)} ();
\end{tikzpicture}}

\begin{abstract}
  We consider the problem of synthesizing control policies that enforce a concentration property on the state distributions of a stochastic system.  We present a formalization of this problem in terms of synthesizing strategies for maintaining an entropy-based objective in Markov Decision Processes (MDPs).  We first show that even relaxed versions of this problem are complexity-theoretically hard. We then present a sound and (conditionally) relatively complete method to verify and synthesize strategies for such entropy objectives. The main challenge is the non-linear nature of such objectives, and our approach addresses this by exploiting and combining ideas from convex duality and invariant synthesis. 
  We also investigate the role of memory and randomization in ensuring entropy objectives.   Finally, we implement our ideas to evaluate our approach empirically on a few illustrative benchmarks.
\end{abstract}

\section{Introduction}

Systems across several application domains can be described in terms of possessing
a ``state'' that determines the probability distribution of the future evolution
of the system. In other words, up to stochasticity, the state faithfully
represents the system. Mathematically, such systems are often modeled as Markov
chains. More generally, when the system can be subjected to external control
inputs, one needs to use the richer formalism of Markov decision processes
(MDPs).

One natural view of Markov chains and MDPs is as stochastic dynamical systems on
the set of states, where the evolution of the system is described by a
stochastic trajectory through the state space. However, there are situations
where it is more natural to consider how the \emph{probability distribution}
over the possible states evolves with time. For a concrete example, consider a
model of drug absorption described in~\cite[Section
III]{DBLP:conf/qest/ChadhaKVAK11}, where the constructed MDP can be interpreted
(informally) as defining the stochastic trajectory of individual drug molecules.
However, the measurable quantity here is the concentration profile of the
molecules in different states, which corresponds to a probability distribution
over these states. Consequently, problems of synthesising controls in such a
setting have to be formulated in terms of the evolution of the essentially
non-stochastic evolution of the probability distributions, and not in terms of
the stochastic evolution of the states. (We note that similar considerations
appear to apply more generally when modeling systems that are analogous to
chemical kinetics: see, e.g., the discussion in \cite{markov-gene-regulatory}.)
This has led to a large body of work that views Markov chains and MDPs as
\emph{distribution transformers}~\cite{DBLP:conf/qest/KorthikantiVAK10}.

A question that arises in the modeling of such systems with MDPs is whether the
probability distributions generated by the MDP are always ``diffuse''. For
example, Chu et al.~\cite{markov-gene-regulatory} only consider probability
distributions in which the probability masses of all the chemical species are
below a threshold. The question of how ``diffuse'' a model is also important in
the context of model selection: Palaniappan et al.~\cite{bio-blaise} consider
low entropy (which is an analytically smoother formalization of ``being diffuse''
than simply using a threshold) as a measure of the quality of their
discretization schemes.

Motivated by such considerations, we consider the problem of deciding whether it is
possible to force an MDP to generate only probability distributions that are
``not diffuse''/``concentrated'', except possibly for the initial few states. Formally, we consider the problem of synthesizing (or verifying) strategies that
ensure that the distribution of the MDP remains far from the uniform
distribution. A natural way to formalize this is to maximize the
Kullback-Leibler (KL) divergence (see e.g.,~\cite{book-cover-thomas}) from the
uniform distribution. Letting $p(s)$ be the probability of the system being in
state $s$, we want
  $$
  \text{maximize} \quad D_{\text{KL}}(p \,||\, u),
  $$
  where $u(s) \defeq \frac{1}{|S|} \forall s \in S$ is the uniform
  distribution. Equivalently, we would like to minimize an entropy objective, namely:
  $$
  \text{minimize} \quad H(p) = -\sum_s p(s) \log p(s) $$
  subject to system dynamics constraints given by the MDP transition equations.

In this paper, we address the following problems: given a finite-state, discrete-time MDP, an entropy objective $H$ (as defined above on the probability distribution), and a
threshold $\gamma$, is there a strategy under which the MDP starting from some given
initial distribution always, except possibly for the first few steps, remains in
a distribution whose entropy is \emph{at most} $\gamma$? Further, can we synthesize such a strategy when it exists? Finally, if a strategy is given, how hard is it to verify that it satisfies the entropy objective as defined above? These problems are special cases of {\em distributional safety} objectives on MDPs (as e.g., addressed by \citet{akshay2023mdps}), where the problem is to check whether the set of distributions under a strategy remains in a safe set defined using $H$.

As mentioned above, the view of MDPs (and Markov chains) as transformers of
distributions has been widely studied over the
years~\cite{DBLP:conf/qest/KorthikantiVAK10,DBLP:journals/tse/KwonA11,DBLP:conf/qest/ChadhaKVAK11,DBLP:journals/jacm/AgrawalAGT15,DBLP:conf/lics/AkshayGV18,0001MS14,DBLP:journals/jcss/DoyenMS19,akshay2023mdps,AkshayCMZ24,DBLP:journals/tac/GaoAXJ24},
for reachability, safety and even omega-regular objectives. However, most of
these studies restrict the target space to be an affine combination of linear
inequalities, and the proof techniques and hardness results, though varied, often depend
on the linearity of the objectives. Our focus and novelty in this paper is to
consider entropy objectives, which means that we have to go beyond
linear or polynomial functions to transcendental functions such as logarithms
and exponentials. This forms the central challenge that we tackle in this paper.
Our contributions are the following:
\begin{itemize}
\item We show that the entropy objective problem is complexity-theoretically
  hard (more precisely NP-hard). Even more interestingly, we show that we cannot
  even approximate the answer in polynomial time, unless there is a randomized
  polytime algorithm for the Boolean Satisfiability problem.  This hardness
  result holds even when the search space for strategies is restricted to the
  space of memoryless strategies (see Preliminaries for the definition). {\em To
    the best of our knowledge this is the first approximation-hardness result
    for entropy and even for general distributional objectives for MDPs.} 
\item {Given the hardness result, we develop a template-based invariant
      synthesis approach to synthesize strategies along with invariants that ensure a low entropy
      objective. To do so, we exploit ideas from convex duality and invariant
      synthesis to reduce the entropy objective to a tractable form.
We show that this formulation leads to a procedure that is {\em sound
        (i.e., any strategy it outputs is guaranteed to satisfy the
        requirements) and relatively complete (i.e., if there is an affine
        invariant and memoryless strategy satisfying the requirement, the
        procedure will detect this)} assuming the decidability of the theory of
      reals augmented with exponentials: this decidability result is known to
      follow from a famous conjecture in transcendental number
      theory~\cite{MacinW96}.\footnote{Our discussion of the results of
        \cite{MacinW96} is based on the more easily available exposition by
        Wilkie~\cite{Wilki97}.} At the same time, we show that we can verify
      whether a candidate strategy and invariant satisfy the entropy objective
      in time polynomial in the input parameters. Finally, we show sufficient
      conditions for when affine invariants suffice for our entropy objectives,
      based on ergodicity of the MDP. }
    \item {We next consider the question of memory and randomization for entropy objectives. We show three results. First, memoryless strategies do not suffice for entropy objectives, i.e., in general we may need memoryful strategies to ensure entropy objectives. Second, we consider the extended notion of {\em distributional} memoryless strategies, that look at the current (and only the current) probability distribution over all states and choose the next action based on the entire distribution, and show that they are necessary and sufficient to ensure entropy objectives. Third, we show that randomized strategies are more powerful than deterministic strategies for entropy objectives.}
    \item While our decision procedure is primarily of theoretical interest, we also implement our approach using a state-of-the-art non-linear
solver applied to the tractable form we obtain.  We show experimental results on a small suite of benchmarks, that validate the theory, as well as demonstrate interesting behaviors and illustrate the potential applicability of these ideas.
\end{itemize}

\paragraph{Related Work}
We are not aware of any work specifically pertaining to keeping the entropy of
the distribution over states {\em low} at every step of the MDP. However, in the
orthogonal setting of path-based semantics, the work of \citet{BiondLNW14},
\citet{DBLP:conf/fsttcs/ChenH14} and \citet{savas2019entropy} considers the
related question of maximizing cumulative path entropy, and the rate of change of path entropy, and suggests that these can be used to model and reason about information leakage in security protocols. In these works, the authors essentially use the fact that the addition to path entropy at one step is only dependent on the state you are at, and hence is a reward function for the MDP, allowing ideas from renewal theory to be used. These ideas allow \citet[Theorem 1]{BiondLNW14} to give a closed form for the entropy rate (in Markov chains), dependent only on the percentage of time spent in a state and the local transition probabilities out of the state (which determine the state's contribution to the entropy
rate). These questions are orthogonal to our problem of minimizing distribution-based entropy objectives, 
and it is not clear how we can reuse any of the renewal theory ideas. For
 instance, \citet[Proposition~1]{savas2019entropy} prove that to maximize
 cumulative path entropy, memoryless strategies are as powerful as strategies
 that keep track of the time, unlike in our setting (see
 \cref{prop:no-state-yes-dist}). In~\citet{DBLP:conf/fsttcs/ChenH14} the authors
 consider maximizing path-based entropy in interval Markov chains and MDPs, and
 use a standard Lagrangian dual to show conditional decidability results via Schanuel's conjecture, similar to what we eventually do. However, the min-max problems are rather different in our approaches, and in fact the approximate variant in their case is shown to be solvable in polynomial time~\cite[Theorem 21]{DBLP:conf/fsttcs/ChenH14}. On the other hand, we show NP-hardness even for our approximation variant. {The core difference is that our definition of entropy is {\em distributional}, i.e., dealing with entropy of the distributions reached as time evolves. As our definition computes the entropy of distributions reached, it is always finite, while the path-based notion of entropy used in \citet{DBLP:conf/fsttcs/ChenH14} as also \citet{BiondLNW14} (see \citet[Definition 7]{DBLP:conf/fsttcs/ChenH14}) often evaluates to infinity (for instance for any Markov chain in which every state has transitions to at least two distinct states). All this makes our problems harder and requires stronger duality results and justifies the completely different template-based approach that we take.}

The idea of template-based synthesis of certificates for safety, termination,
reach-avoidance and even omega-regularity objectives for probabilistic systems
and programs has been explored in several recent works (see, e.g., \cite{chatterjee2025polyqent,ChatterjeeFG16,CNZ17,akshay2023mdps,AkshayCMZ24,DBLP:conf/cav/AbateGR24}). Similar ideas have also been widely explored in control theory in several works e.g., ~\cite{DBLP:journals/tac/PrajnaJP07, DBLP:conf/cav/SankaranarayananT11, DBLP:conf/hybrid/AnandM0Z22}. However, to the best of our knowledge none of these approaches work with entropy objectives. 

At a high level, we wish to relate template-based synthesis methods, especially synthesis methods
based on affine or linear templates,  with the use of
\emph{relaxations} in combinatorial optimization.  The broad idea there is to
relax the combinatorial control variables to allow them to take continuous
values, so that the resulting optimization problem becomes tractable (often a
linear program, or more generally, a tractable convex program).  Of course, the
resulting relaxed solution may no more be a valid solution to the original
combinatorial problem, and one has to then work to translate the relaxed
solution back to the original combinatorial domain. Several deep and important
results in the theory of approximation algorithms for combinatorial problems are
based on analyzing the capabilities and limitations of such
relaxations~\cite{vaziraniApproximationAlgorithms2003}.  The use of templates is
motivated by a similar desire to make the search problem tractable; however,
here, the search space is shrunk rather than enlarged.  We consider the question
of developing a theory of the relative power of classes of templates (analogous
to the theory of integrality gaps for relaxations of combinatorial optimization
problems) to be an important direction for future work.

{The structure of the paper is as follows. We start with preliminaries in Section~\ref{sec:prelim}. We formulate our problem statements in Section~\ref{sec:entropy}, along with illustrative examples of entropy objectives. In Section~\ref{sec:hardness} we provide our hardness results, followed by our invariant-synthesis based algorithm in Section~\ref{sec:algo}. Section~\ref{sec:random-memory} deals with the question of memory and randomization. We implement our approach and demonstrate experimental results in Section~\ref{sec:experiments} and conclude in Section~\ref{sec:conclusion}. A few technical details and more experimental results and insights are presented in the Appendix. 
}

\section{Preliminaries}
\label{sec:prelim}
Given a finite set \(\Omega\), we denote by \(\Delta(\Omega)\) the set of probability
distributions on \(\Omega\).  For a probability distribution \(\mu\) on
\(\Omega\), we denote by \(H(\mu) \defeq -\sum_{x \in \Omega}\mu(x)\log \mu(x)\) the \emph{entropy}
of \(\mu\).  It is well known that the entropy of a probability distribution
\(\mu\) is a measure of the \emph{uncertainty} in a sample drawn from \(\mu\); in
particular, its highest value of \(\log \abs{\Omega}\) is achieved for the uniform
distribution, and its lowest value of \(0\) at \emph{trivial} probability
distributions, i.e., those supported on a single point in \(\Omega\).
A \emph{Markov decision process} (MDP) \(M = (S, \act, \delta)\) is a triple
comprising a set \(S\) of \emph{states}, a set \(\act(s)\) of allowed
\emph{actions} for each state \(s \in S\), and a \emph{transition function}
\(\delta\) that maps every \emph{state-action pair} \((s, a)\) (with \(s \in S\) and
\(a \in \act(s)\)) to a probability distribution \(\delta(s, a)\) on \(S\).  A
\emph{strategy} \(\pi\) for an MDP is a function that maps, for every time
\(t\), every state-action history sequence
\((s_0, a_0), (s_1, a_1), (s_2, a_2), \dots, (s_{t-1}, a_{t-1}), s_t\) to a
distribution on the set \(\act(s_t)\) of actions allowed for the current state
\(s_t\).  A strategy \(\pi\) is said to be \emph{deterministic} if its output is
always a trivial/Dirac probability distribution.  It is said to be \emph{memoryless}
if its output depends only on the current state \(s_t\).  Given an MDP \(M\), a
strategy \(\pi\) for \(M\), and an initial distribution \(\mu_0\) on the states
\(S\) of \(M\), we denote by \(M^{\pi}(\mu_0, t)\) the resulting distribution of the
state \(s_t\) at time \(t\) when the MDP \(M\) uses the strategy
\(\pi\), starting with state \(s_0\) sampled according to \(\mu_{0}\).  Also, \(M^{\pi}(\mu_0, 0) = \mu_0\) for any strategy \(\pi\). We let \memoryless{} denote the set of memoryless strategies.

\section{Entropy Objectives}
\label{sec:entropy}
Our objective in this work is to synthesize strategies that ensure that the
distribution of the MDP remains far from the uniform distribution.  To formalize
this, we consider a natural notion of distance on probability distributions: the
\emph{Kullback-Leibler} (KL) divergence. The (asymmetric) KL divergence
\(D_{KL}(P \Vert Q)\) between probability distributions \(P\) and \(Q\) on the same
sample space \(\Omega\) is defined as
\begin{equation}
D_{KL}(P \Vert Q) \defeq \sum_{\omega \in \Omega}P(\omega)\log\inp{\frac{P(\omega)}{Q(\omega)}}.
\end{equation}
In particular, if \(\mathcal{U}\) denotes the uniform distribution on \(\Omega\), then
\begin{equation}
D_{KL}(P||\mathcal{U})  = \log \abs{\Omega} - H(P),
\end{equation}
where \(H(P)\) is the entropy of the probability distribution \(P\).  Thus, we
formulate our goal as synthesizing strategies that ensure that the entropy of
the state distribution of the MDP remains low. Further, there could be an
initial period where the system has not yet settled into low entropy, so we only
require the entropy to remain low after an initial warm-up period.  We now
formalize these notions.  Specifically, given the transcendental nature of the
entropy objective, having infinite precision is neither practically
realistic, nor feasible.  Hence, we consider the following robust version.
\begin{problem}[\textbf{\initunfair}]\label{prob-init-unfair}
  \textbf{INPUT}: An MDP \(M = (S, \act, \delta)\), a starting probability
  distribution \(\mu_0 \in \Delta(S)\), an accuracy parameter
  \(\epsilon \in (0, 1)\), and a warm-up parameter $0\leq K\in \mathbb{N}$. \textbf{OUTPUT}: A \(\pm\epsilon\)-additive approximation to the minimum, over all
  strategies \(\pi\), of the maximum value over all times \(t \geq K\) of the entropy
  of \(M^{\pi}(\mu_0, t)\).
  
  That is, we want to compute a value $v$ such that:
  $$|\min_{\pi} \max_{t \geq K} H(M^{\pi} (\mu_0,t))-v|\leq \epsilon.$$
  
\end{problem}

In the next section, we show that the following simpler decision version of
\initunfair{} is already NP-hard.
\begin{problem}[\textbf{\(\tau\)-\apxinitunfairDec}]\label{prob-approx-init-unfair-dec} Let \(\tau > 0\) be a fixed
  approximation parameter. \textbf{INPUT}: An MDP \(M = (S, \act, \delta)\), a
  starting probability distribution \(\mu_0 \in \Delta(S)\), a warm-up parameter $K\geq0$ and a rational threshold
  \(\gamma > 0\). \textbf{OUTPUT}: \yes{}, if there is a strategy \(\pi\) for which the maximum
  value over all times \(t \geq K\) of the entropy of \(M^{\pi}(\mu_0, t)\) is at most
  \(\gamma\); \no{} if for any strategy \(\pi\), this maximum value is at least
  \(\gamma + \tau\). Otherwise, the algorithm is allowed to output either \yes{} or
  \no{}.

  When \(\tau = 0\), we obtain an ``exact'' version of the problem, which we refer
  to simply as \initunfairDec{}.
  \end{problem}
{We will also be interested in the synthesis variant of the above problems, which asks to synthesize the strategy $\pi$ that gave rise to the \yes{} answer in each case.}

Before proceeding to our results, we observe that these questions are of a
distributional nature, i.e., they reason about the probability distributions
reached over time. In particular, they can be seen as a special case of
\emph{distributional safety}~\cite{akshay2023mdps}: given a ``safe'' subset
\(S \subseteq \Delta(\Omega)\) of distributions over the states, we say that a strategy
\(\pi\) for an MDP \(M\) with initial distribution \(\mu_0\) is
\emph{\(S\)-safe} if at all times \(t \geq K\), it holds that
\(M^{\pi}(\mu_0, t) \in S\) (equivalently, the initial distribution \(\mu_0\) is then
said to be \(S\)-safe with respect to the strategy \(\pi\) after $K$ steps). Note
that we do \emph{not} {necessarily} require that \(\mu_0\) itself should belong to the safe set:
we only require safety from the $K^{th}$ action onwards.
The main difference of the above formulation from earlier works on
{distributional safety such as~\cite{akshay2023mdps} is that}
those deal with safety queries over affine (and sometimes, polynomial) objectives, while the
entropy objective as we defined above is intrinsically non-linear and not even
captured by polynomials. 

Let us next give two examples that illustrate entropy objectives in our problem setting.

\begin{example}
\label{eg:MDPM1} Consider the MDP $M_1$ in Figure~\ref{fig:MDPM1} with three
states, taken from the running example of ~\cite{akshay2023mdps}. Let
$\mu_0=(1/2,1/3,1/6)$ be the initial distribution, in which state A has
probability $\frac{1}{2}$, state B has probability $\frac{1}{3}$, and state C has probability $\frac{1}{6}$. 
\begin{figure}[h!]
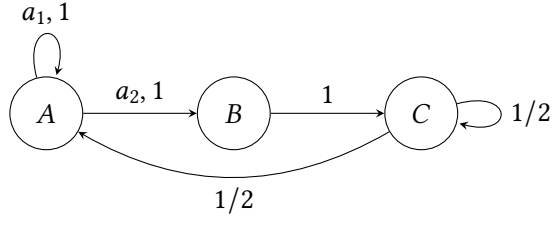

\begin{center}
\runningexamplefig

\end{center}
\caption{MDP $M_1$ taken from~\cite{akshay2023mdps}}
\label{fig:MDPM1}
\end{figure}

If we wish to minimize entropy, it is {intuitive} that we should choose action
$a_{1}$ always: as the initial probability mass at $A$ exceeds $\frac{1}{e}$,
any additional probability mass will only decrease its contribution to the
entropy. But what is the value of $v$ (as defined in
Problem~\ref{prob-init-unfair}) for different values of $K$ with
$\epsilon=0.01$? With some computation, one can show that for $K=0$,
$v\geq1.01$, while for $K=1$, $v\leq0.68$ and for $K=2$, $v\leq0.60$. Further, if we fix
$\gamma=1$, $\tau=0$, $K=1$, numerical experiments using our linear invariant approach
(reported in the appendix) also establish that the answer to
Problem~\ref{prob-approx-init-unfair-dec} (i.e., whether there exists a strategy
for which the max value of entropy remains below 1) is \yes{}, and that this is
achieved by the strategy that always chooses $a_{1}$.
\end{example}

In the above example, it was easy to identify a strategy that minimizes entropy
(increases concentration) even if the value of the strategy required some work
to compute or bound; but this is not always the case. For instance, in the
following, it is \emph{a priori} unclear what strategy minimizes the entropy or
ensures that it is below a certain threshold.

\begin{example} 
\label{eg:MDPM2} Consider the 4-state MDP $M_2$ given in Figure~\ref{fig:MDPM2}, with initial distribution $\mu_0=  (0.2, 0.4, 0, 0.4)$.

\begin{figure}[h]
\begin{center}
\begin{tikzpicture}[>=stealth, node distance=1.5cm and 1.5cm,scale=0.5]
\node[state] (D) {$D$};
\node[state] (B) [below left=of D] {$B$};
\node[state] (C) [below right=of D] {$C$};
\node[state] (A) [below right=of B] {$A$};

\draw[->] (A) -- node[left] {$a,1$} (B);
\draw[->] (A) -- node[right] {$b,1$} (C);

\draw[->] (B) to[bend left=40] node[left] {$4/5$} (D);

\draw[->] (B) to[bend left=15] node[above] {$1/5$} (C);

\draw[->] (C) to[bend right=40] node[right] {$1/2$} (D);

\draw[->] (C) to[bend left=15] node[below] {$1/2$} (B);

\draw[->] (D) -- node[left] {$1/3$} (B);

\draw[->] (D) -- node[right] {$2/3$} (C);

\end{tikzpicture}
\end{center}
\caption{MDP $M_2$}
\label{fig:MDPM2}
\end{figure}
In this example, starting from $\mu_0$, with  $K=0$, suppose $\gamma=1.092$, what would be the strategy that ensures that the \initunfairDec{} problem answers \yes{}? With some computation, one can check 
that a (randomized) strategy that selects action \(a\) with probability
    \(7/10\) indeed ensures this, even though none of the two deterministic
    strategies do.
\end{example}

{In the rest of the paper, our goal is to address the above defined problems. We first establish complexity-theoretical hardness results, that work
  even for the approximate variant of the problem. We then provide an affine
  invariant-synthesis based algorithm with soundness and relative completeness
  guarantees.  Finally, we present a sufficient condition under which such
  invariants suffice. }

\section{Hardness}
\label{sec:hardness}
We now show that the approximate decision version of \initunfair{} is NP-hard, under randomized reductions, even for small enough \emph{fixed} values of the approximation parameter. More precisely, we prove the following theorem. 
\begin{theorem}[\textbf{\apxinitunfairDec{} is NP-hard}]\label{thm-np-hard}
  There exists a positive constant \(\tau\) such that if there is a polynomial-time
  algorithm for \(\tau\)-\apxinitunfairDec{}, then NP = RP.\footnote{More
    concretely, NP = RP means that there is a randomized polynomial time
    algorithm for SAT that is always correct on satisfiable inputs and correct
    with probability at least \(2/3\) (over the randomness of the algorithm) on
    unsatisfiable inputs~ (see for instance, \cite[Chapter 7.3]{aroraB09} for more details).}
\end{theorem}

\begin{remark}
  As an immediate corollary, we see that \(\initunfairDec\), the exact
  version of the problem, is also NP-hard: a polynomial time algorithm for it
  would imply randomized polynomial time algorithms for SAT.
\end{remark}

In the proof of \cref{thm-np-hard}, we will need the following computation.
\begin{observation}[\textbf{Entropy increment by mass transfer}]
  \label{lem-entropy-increment}
  Let \(\mu\) be a probability distribution on a finite set \(S\). Let
  \(\epsilon > 0\) and suppose that \(a, b \in S\) are such that \(\mu(b) = 0\) and
  \(\mu(a) > \epsilon\).  Consider the probability distribution \(\mu'\) obtained from
  \(\mu\) by transferring an \(\epsilon\) mass from \(a\) to \(b\).  Formally,
  \(\mu'(k) = \mu(k)\) for \(k \in S \setminus \inb{a, b}\),
  \(\mu'(b) = \epsilon\) and \(\mu'(a) = \mu(a) - \epsilon\).  Then
  \begin{align*}
    H(\mu')
    &= H(\mu) + \epsilon\log\inp{\frac{\mu(a)}{\epsilon} - 1}
      -\mu(a)\log\inp{1 - \frac{\epsilon}{\mu(a)}}\\
    &\ge  H(\mu) + \epsilon\log\inp{\frac{\mu(a)}{\epsilon}}.
  \end{align*}
\end{observation}

\begin{proof}
  The equality follows by a direct computation.  For the inequality, we use
  \(0 < \epsilon < \mu(a)\) to get that
  \begin{equation}
    \label{eq:6}
    \mu(a)\log\inp{1 - \frac{\epsilon}{\mu(a)}} \leq \epsilon\log\inp{1 - \frac{\epsilon}{\mu(a)}},
  \end{equation}
  and substitute this in the equality.
\end{proof}

We now proceed with the proof of \cref{thm-np-hard}.
\begin{proof}[Proof of \cref{thm-np-hard}]
  We start with the following hardness result of \citet[Section 2]{Trevi01} for
  MAX-SAT with bounded occurrences of variables. Consider the special case of
  SAT where each clause has exactly three literals and each variable occurs in
  at most \(B\) clauses, and consider for these instances the problem of
  deciding whether (a) all the clauses are satisfiable, or (b) at most a 9/10
  fraction of the clauses are satisfiable.  Then, there exists a positive
  integer \(B\) such that if there is a polynomial time algorithm for this
  problem, then NP = RP.  We will show that there exists a \(\tau\) (depending upon
  \(B\)) such that a polynomial time algorithm for \(\tau\)-\apxinitunfairDec{}
  implies a polynomial time algorithm for Trevisan's version of MAX-SAT.
  Combined with Trevisan's result, this proves \cref{thm-np-hard}.

  \newcommand{\tV}{\ensuremath{\tilde{V}}} \newcommand{\var}[1]{\ensuremath{\mathrm{var}\inp{#1}}} Consider an instance \(\phi\) of MAX-SAT of the above kind.  Let \(C\) and
  \(V\) denote respectively the sets of clauses and variables of \(\phi\), and
  define \(L \defeq \bigcup_{v \in V}\inb{v, \lnot v}\).  We construct an MDP
  \(M_{\phi}\) whose state space \(\Omega = C \cup \tV \cup L \), where \(\tV\) is a marked
  copy of \(V\) (so as to distinguish its elements from those of \(L\)) defined
  as \(\tV \defeq \inb{\var{v} \st v \in V}\).  The states in \(L\) are absorbing
  states: the only allowed action from \(s \in L\) is to loop back to \(s\) with
  probability \(1\).  For each \(v \in V\), there are two allowed actions for
  \(\var{v} \in \tV\): one is to go to \(v \in L\) with probability \(1\), the other
  is to go to \(\lnot{v} \in L\) with probability one.  For each clause
  \(c = (\ell_1 \lor \ell_2 \lor \ell_3)\) in \(C\), there are three allowed actions: for
  \(i \in \inb{1, 2, 3}\) the \(i\)th action is to go to \(\ell_i \in L\) with
  probability one.  Since all actions in the MDP \(M_{\phi}\) are deterministic, we
  will denote, by a slight abuse of notation, each action by its (unique) goal
  state.  We also denote the number of clauses as \(m\) and the number of
  variables as \(n\).  Note also that since each of the \(m\) clauses has at
  most three literals, the total number of literals (and hence variables) that
  can occur in the formula is at most \(3m\), so that \(n \leq 3m\).

  Let \(B\) be a positive integer as above; without loss of generality, we
  assume \(B \geq 10\) for simplicity.  Let
  \[\alpha \defeq 1 - \exp(-100B^2) \geq 1/2 \qquad \text{and} \qquad
    \beta \defeq 1- \alpha\] be positive constants.  Consider the initial distribution
  \(\mu_0\), supported on \(S \cup \tV\), which assigns weight
  \(\frac{\alpha}{n}\) to each element of \(\tV\) and weight \(\frac{\beta}{m}\) to each
  element of \(C\).

  A simple consequence of the construction of \(M_{\phi}\) is that for any strategy
  \(\pi\), the distribution \(M_{\phi}^{\pi}(\mu_0, 1)\) after the first action remains
  the same at all subsequent times.  Thus, we only need to study the entropy of
  \(M_{\phi}^{\pi}(\mu_0, 1)\).

  Consider first the case when \(\phi\) is satisfiable.  In this case, fix a
  satisfying assignment \(A\) and consider a strategy \(\pi_A\) which for each
  \(v \in V\) chooses for \(\var{v}\) the action \(v\) if \(A(v) = 1\) and the
  action \(\lnot{v}\) if \(A(v) = 0\).  Similarly, for each \(c \in C\), it picks
  a literal of \(c\) satisfied by \(A\) and picks the corresponding action.
  Note that \(\pi_A\) is a {\em memoryless} strategy.  By construction,
  \(M_{\phi}^{\pi_A}(\mu_0, 1)\) is supported on a subset of \(L\) of size exactly
  \(n\), and each element of the support has probability at least
  \(\frac{\alpha}{n}\).  Thus, if \(\phi\) is satisfiable, then there exists a strategy
  \(\pi\) (namely \(\pi_A\)) such that \(M_{\phi}^{\pi}(\mu_0, 1)\) (and, therefore,
  \(M_{\phi}^{\pi}(\mu_0, t)\) for all \(t \geq 1\), as noted above) has entropy at most
  \(\log(n/\alpha)\).

  Next, consider the case when at most \(9m/10\) of the clauses in \(\phi\) are
  satisfiable.  Our goal is to show that in this case, for every strategy $\pi$,
  the entropy of \(M_{\phi}^{\pi}(\mu_0, 1)\) is at least
  \(\tau + \log (n/\alpha)\), for some positive constant \(\tau\).  We first observe that
  it is enough to prove this lower bound for \emph{deterministic} strategies.

  To see this, note that when \(\pi\) is randomized,
  \(M_{\phi}^{\pi}(\mu_0, 1)\) is a convex combination of various
  \(M_{\phi}^{\pi'}(\mu_0, 1)\) where \(\pi'\) varies over a set of deterministic
  strategies, so that, by the concavity of the entropy, the entropy of
  \(M_{\phi}^{\pi}(\mu_0, 1)\) is at least as much as the entropy of
  \(M_{\phi}^{\pi'}(\mu_0, 1)\) for some deterministic strategy \(\pi'\).  Thus, in order
  to show that for every strategy $\pi$, the entropy of
  \(M_{\phi}^{\pi}(\mu_0, 1)\) is at least \(\tau + \log (n/\alpha)\), for some positive
  constant \(\tau\), it is enough to show that for every \emph{deterministic}
  strategy $\pi$, the entropy of \(M_{\phi}^{\pi}(\mu_0, 1)\) is at least
  \(\tau + \log (n/\alpha)\).

  Consider now an arbitrary deterministic strategy \(\pi\) and denote by
  \(L_{\pi}\) the set of literals in \(L\) selected by \(\pi\) for the states in
  \(\tV\).  Since at most 9/10 fraction of the clauses in \(\phi\) can be satisfied
  by any assignment, it follows that for at least \(m/10\) clauses in \(C\),
  \(\pi\) must select a literal that is not in \(L_{\pi}\).  Further, since each
  variable appears at most \(B\) times in \(\phi\), at least \(m/(10B)\) literals
  in \(L \setminus L_{\pi}\) must be selected by these (at least) \(m/10\) unsatisfied
  clauses. Let \(\Gamma\) denote the set of these chosen literals in
  \(L \setminus L_{\pi}\), so that \(\abs{\Gamma} \geq m/(10B)\).  Let
  \(\mu_1\) denote the probability distribution \(M_{\phi}^{\pi}(\mu_0, 1)\) induced by
  \(\pi\) at time \(1\).  Then, the probability weight assigned by \(\mu_1\) to each
  literal in \(\Gamma\) lies in the interval \([\beta/m, B\beta/m]\).  On the other hand, the
  probability weight assigned by \(\mu_1\) to each literal in \(L_{\pi}\) is at
  least \(\alpha/n\).  Since \(m \geq n /3\), our choice of \(\alpha\) and
  \(\beta\) implies that \(\alpha/n \geq B\beta/m\).  Using these facts, we now show that
  \(H(\mu_1) \geq \log(n/\alpha) + 3\beta\).

  Towards this, consider the probability distribution \(\tilde{\mu}\), supported
  only on \(L_{\pi}\), obtained from \(\mu_{1}\) by transferring the probability
  weight of each literal \(\ell \in \Gamma\) to the unique literal in
  \(L_{\pi}\) corresponding to the negation of \(\ell\).  Then, for each
  \(\ell \in L_{\pi}\), \(\tilde{\mu}(\ell) \in [\alpha/n, \alpha/n + B\beta/m]\).  It follows that
  \begin{equation}
    \label{eq:1}
    H(\tilde{\mu})  \geq
    \log(n/\alpha) - \log(1 + (B\beta/m)/(\alpha/n)),
  \end{equation}
  which is at least \(\log(n/\alpha) - 3B\beta/\alpha\) (since
  \(m \geq n / 3\)).  Further, by the construction of \(\tilde{\mu}\), \(\mu_1\) can be
  obtained from \(\tilde{\mu}\) by transferring, sequentially for each literal
  \(\ell\) in \(\Gamma\), probability mass \(\mu_1(\ell)\) from the literal
  \(\lnot \ell \in L_{\pi}\) to the literal \(\ell\).  Applying the entropy increment
  estimate of \cref{lem-entropy-increment} to each such transfer, we see that
  each such transfer increases the entropy by at least
  \(\mu_1(\ell)\log(\tilde{\mu}(\lnot \ell)/\mu_1(\ell))\), which is at least
  \(\frac{\beta}{m}\cdot\log\inp{\frac{\alpha m}{nB\beta}}\) since as argued above,
  \(\mu_1(\ell) \in [\beta/m, B\beta/m]\) and
  \(\tilde{\mu}(\lnot \ell) \geq \alpha/n\).  Since there are
  \(\abs{\Gamma} \geq m/(10B)\) transfer steps and we have \(m \geq n /3\), we therefore
  get
  \begin{align}
    H(\mu_1)
&\geq H(\tilde{\mu}) + \frac{\beta}{10B}\log\inp{\frac{\alpha}{3B\beta}}\label{eq:7}\\
    &\geq \log(n/\alpha) + \frac{\beta}{10 B} \cdot \insq{\log \chi - 60B^2}\\
    &\geq \log(n/\alpha) + 3\beta.
  \end{align}
  (Here, \(\chi \defeq \frac{\alpha}{3B\beta} \geq \exp(90B^2)\) with our choice of
  \(\alpha\) and \(\beta\), and we also use \(\alpha \geq 1/2\) to get the second inequality.)

  We thus see that when at most \(9/10\) fraction of the clauses in \(\phi\) are
  satisfiable, \(M_{\phi}^{\pi}(\mu_0, 1)\) has entropy at least
  \(\log(n/\alpha) + 3\beta\) (where \(\beta \) as chosen above is an absolute constant). On
  the other hand, as shown above, when \(\phi\) is satisfiable, there is a strategy
  \(\pi\) which ensures that the entropy of \(M_{\phi}^{\pi}(\mu_0, t)\), for every
  \(t \geq 1\), is at most \(\log(n/\alpha)\). Using standard methods for computing
  approximations to the logarithm, we can compute a rational threshold
  \(\gamma \in [\log(n/\alpha), \log(n/\alpha) + \beta]\) in time \(O(\log n)\) (note that both
  \(\alpha\) and \(\beta\) are absolute constants). Further, given \(\phi\),
  \(M_{\phi}\) can be constructed in time polynomial in \(m\) and \(n\).
  \(M_{\phi}\) \(\mu_0\) and \(\gamma\) thus give us a polynomial-time constructible
  instance of \(\tau\)-\apxinitunfairDec{} (with \(\tau = \beta\) and \(K = 1\)) that has
  answer {NO} when at most \(9/10\)-fraction of the clauses in \(\phi\) are
  satisfiable, and {YES} when \(\phi\) is satisfiable. Combined with Trevisan's
  hardness result mentioned above, the above reduction establishes that a
  polynomial time algorithm for \(\tau\)-\apxinitunfairDec{} (with
  \(\tau = \beta\) and \(K = 1\)) implies that NP = RP.
\end{proof}

\begin{remark}
  Consider the memoryless version of \apxinitunfairDec{}, where we constrain the
  strategy \(\pi\) to be memoryless.  Since the strategy \(\pi_A\) constructed in
  the above proof in the \yes{} case is memoryless, it follows that even the
  memoryless version of \apxinitunfairDec{} is NP-hard (under randomized
  reductions).  Further, the proof above can be easily modified to
    prove hardness for any \(K > 1\) as well.
\end{remark}

\section{An invariant synthesis based algorithm}
\label{sec:algo}
As described in the previous section, even the decision version of our
synthesis problem, even for \emph{memoryless} strategies, is hard. A fruitful paradigm for attacking such hard problems has been through the construction of sound but possibly incomplete methods via \emph{invariant} and template-based synthesis. This approach has been used in different settings including program correctness~\cite{ColonSS03,DBLP:conf/pldi/Chatterjee0GG20}, probabilistic termination~\cite{CNZ17,BatzCJKKM23}, control theory~\cite{DBLP:journals/tac/PrajnaJP07,DBLP:conf/cav/SankaranarayananT11}. For distributional objectives in MDPs this broad technique has been adapted in~\cite{akshay2023mdps,AkshayCMZ24} for the special case of affine objectives. Our goal in this section is to lift this approach to entropy objectives that are non-linear and even non-polynomial.

Instead of attempting a general description of the idea of invariants, we
describe the instantiation of the idea in the context of strategy synthesis for
enforcing distributional objectives in MDPs~\cite{akshay2023mdps}.  Let
\({x^{(t)}}\) denote the probability distribution of the MDP at time \(t\).  Let us
suppose that our objective is to enforce that for at all times \(t \geq K\) (where
\(K\) is the warm-up parameter), a predicate \(\predicate({x}^{(t)})\) holds.  For our
problem, the predicate \(\predicate = \predicate_\gamma\) is parameterized by a threshold
\(\gamma > 0\), and takes the following form:
\begin{equation}
  \label{eq-safety-predicate}
  \predicate(x) = \predicate_{\gamma}(x) \defeq H(x) \leq \gamma,
\end{equation}
where \(H\) is the entropy function.

Given an initial distribution \(\mu_0\) of the MDP, the idea is to search together
for a memoryless policy \(\pi\) and another predicate \(I\) (called an \emph{invariant}),
which have the following properties:
\begin{enumerate}
\item \textbf{(Initialization)} \(I(M^{\pi}(\mu_0, K))\) holds;\label{item-inv-begin}
\item \textbf{(Induction)} \(\forall x, I(x) \implies I(M^{\pi}(x,
  1))\) \label{item-inv-induct}; and
\item \textbf{(Invariant implies safety)} \(\forall x, I({x}) \implies
  \predicate(x)\). \label{item-inv-safe}
\end{enumerate}
Note that if we find a memoryless strategy \(\pi\) and an invariant \(I\) satisfying the above
properties, then it follows immediately (by induction) that
\(\predicate(M^{\pi}(\mu_0, t))\) is satisfied for all times \(t \geq K\).

\subsection{Affine Invariants} 
The problem then reduces to automating the search
for such invariants.  Once again, to make an automated search feasible, the
popular paradigm is to choose one of the simplest possible kinds of invariants:
an affine invariant.  In this one searches over those \(I\) that can be
represented as a conjunction of affine inequalities.  In other words, we fix a
number \(m\) of inequalities (we call this
number the \emph{template size} of the invariant \(I\)), and if \(n\) is the
number of states of the MDP, then in terms of an \(m \times n\) real matrix
\(A\) and a real vector \(b \in \R^m\), we define
\begin{equation}
  \label{eq-invariant}
  I(x) \defeq
  \begin{cases}
    Ax \leq b,\\
    \vec{1}^{\top}x = 1, \text{ and }
     x \geq 0.
  \end{cases}
\end{equation}
i.e., \(I(x)\) holds if and only
if the inequalities \(Ax \leq b\) hold coordinate-wise and \(x\) is a probability
distribution.

It is clear that for such invariants, the initialization constraint
(\cref{item-inv-begin}) is simply the feasibility question for a system of
polynomial inequalities (which are linear in \(A\) and \(b\), but possibly
non-linear, when \(K > 1\), in the variables representing \(\pi\)).  We now write
the defining constraints in the inductive condition (\cref{item-inv-induct}) for
such an affine invariant in the above notation.  This leads to a conjunction of
constraints of the form
\begin{equation}
  \label{eq-impl-linear}
\Phi = \Phi(Q, p, c, d) \defeq  \forall x (Qx \leq p \implies c^\top x \leq d),
\end{equation}
where \(Q, p, c\) and \(d\) depend upon \(A, b\) and the variables representing
the strategy \(\pi\), and the antecedent \(Qx \leq p\) always denotes the invariant
constraint \(I(x)\) from \cref{eq-invariant} (the equality constraint in
\cref{eq-invariant} is handled by the usual trick of replacing an equality by a
pair of inequalities). {Such universally
quantified implication constraints cannot be easily handled by standard solvers. Thus, several invariant synthesis approaches, including those of \citet{ColonS01,ChatterjeeGMZ22,chatterjee2025polyqent} and \citet{akshay2023mdps} convert the above quantified linear constraints into a simple existential constraint via an application of linear programming duality in the form of \emph{Farkas' lemma}, as stated below.}
\begin{theorem}[\textbf{Farkas' lemma}]\label{lem-farkas-linear}
  Let \(Q \in \R^{\ell \times n}, c \in \R^n, d\in \R\) and \(p \in \R^{\ell}\) be fixed.  Suppose
  that the system \(Qx \leq p\) is feasible (i.e., some \(x \in \R^{n}\) satisfies
  it).  Then the constraint \(\Phi(Q, p, c, d)\) in \cref{eq-impl-linear} is
  equivalent to the constraint
  \begin{equation*}
    \dualc = \dualc(Q, c, p, d) 
    \defeq \exists y (y \geq 0 \land Q^{\top}y = c \land p^{\top}y \leq d).
  \end{equation*}
\end{theorem}
Note that the initialization condition in \cref{item-inv-begin} implies that the
equivalence in \cref{lem-farkas-linear} applies to each instance of
\cref{eq-impl-linear} arising from \cref{item-inv-induct}: this is because for
these instances, the requirement \(Qx \leq p\) is the same as \(I(x)\), which, by
the initialization condition, has the feasible solution
\(x = M^{\pi}(\mu_0, K)\).  We can therefore apply this equivalence to transform all
the constraints of the form described in \cref{eq-impl-linear} arising from
\cref{item-inv-induct} into a single feasibility question about a system of
polynomial inequalities over the variables \(A, b, \pi\) and the freshly
introduced auxiliary variables \(y\) (a separate set of these auxiliary
variables is introduced for each of the translated constraints).  Note further
that this system is actually a system of non-strict \emph{linear} inequalities
in the auxiliary variables when \(A, b,\) and \(\pi\) are fixed.

It remains to handle the important ``invariant \(\implies\) safety'' constraint
given in \cref{item-inv-safe}.  In  \cite{akshay2023mdps}, the safety predicate
\(\predicate\) was itself given as a conjunction of finitely many linear inequalities.
Thus, in their setting, the constraints arising from \cref{item-inv-safe} are
also of the form given in \cref{eq-impl-linear}, and the above process was sufficient.

\subsection{
Entropy-based Safety Constraints}
Our point of departure from this earlier setting 
is that since our safety predicate \(\predicate\) is no more described in terms of a finite set of linear
inequalities (see \cref{eq-safety-predicate}), we cannot use the form of
Farkas' lemma given in \cref{lem-farkas-linear} {(or the even more powerful Handelman's Theorem~\cite{handelman1988representing} used for polynomial constraints)} to carry out the above
translation.  We get around this by exploiting the fact that our entropy objective
is convex, and take recourse to general convex duality.  In particular, we use
the following result from the well-known duality theory of maximum-entropy
programs.
\begin{proposition}[\textbf{see, e.g., \citet[p.~228]{BV04}}]\label{prop-entropy-dual}
  Let \(Q \in \R^{\ell \times n}\) and \(p \in \R^\ell\) be fixed.  Let
  \(\D \defeq \inb{x | x \geq 0} \subseteq \R^n\) be the domain of the concave function
  \(H(x) \defeq -\sum\limits_{i=1}^nx_i\log x_i\).  Suppose that there is an
  \(x_{0} \in \mathcal{D}\) satisfying \(Qx_0 \leq p\) and
  \(\vec{1}^{\top}x = 1\). Then the optimization program
  \begin{equation}
    \label{eq-ent-cons}
    \begin{aligned}
      \sup\limits_{x \in \D} H(x)
      \textup{ subject to }  Q x \leq p \textup{ and } \vec{1}^{\top}x = 1
    \end{aligned}
  \end{equation}
  has the same optimal value as the program
  \begin{equation}
    \label{eq-ent-dual}
    \inf\limits_{\lambda \geq 0} \quad g(\lambda) \defeq p^{\top}\lambda + \log \inp{\sum\limits_{i=1}^n\exp(-\lambda^{\top}Q_i)}.
  \end{equation}
  Here \(Q_i\) denotes \(i\)th column of \(Q\).  Further, \(g\) is a convex
  function of \(\lambda\).
\end{proposition}
\begin{remark}
  The above result is usually stated under the additional hypothesis that the feasible point \(x_0\) has all its entries {\em strictly} positive (this is the version given, e.g., by \citet{BV04}), in which form it is proved using Slater's
  constraint qualification.  
  
  {However, this additional hypothesis is not required for
    maximum-entropy programs with linear constraints, like the one in
    \cref{eq-ent-cons} above. In this case, it is sufficient that the primal
    program has at least one feasible solution, and this can be shown for
    instance using the general duality results of \citet[Theorem
    30.4]{rockafellar70}. We believe this result to be folklore, but for the
    sake of completeness, and due to its independent interest, we provide a
    proof of the above proposition in
    Appendix~\ref{app:proofs}.}\end{remark}

\subsection{
A Sound and (Conditionally) Relatively Complete Procedure} 
We now have all the technical ingredients needed to develop our algorithmic procedure. We combine the above strong duality result with ideas from invariant synthesis to obtain a sound and (conditionally) relatively complete
synthesis algorithm for \initunfair{} (in the following, we use the notation
introduced in its definition in \cref{prob-init-unfair}) {that we call \textsc{AlgoEnt}}.  We frame the problem as that of finding the smallest possible threshold \(\gamma\) for which we can find a
strategy \(\pi\) that ensures that the predicate 
\begin{equation}
\predicate_{\gamma}(x) \defeq H(x) \leq \gamma
\end{equation}
holds for \(x = M^{\pi}(\mu_0, t)\) for all \(t \geq K\). Let \(n\) be the number of
states in the MDP \(M\).  \\
{\bf Step 1.} We start by choosing a positive integer parameter \(m\), the
\emph{template size} of the affine invariant \(I\), and consider the affine
invariant \(I(x)\) described above in \cref{eq-invariant}. We create the
following set of \emph{template variables}: an \(m \times n\) matrix \(A\) along with
an \(m\) dimensional vector \(b\) of real valued variables (in order to encode
the invariant \(I(x)\)), and real-valued variables \(\pi(s, a)\) for each state
action pair \((s, a)\) in the MDP. We collectively denote these latter
variables, with a slight abuse of notation, by \(\pi\), and they encode the
memoryless strategy \(\pi\) that we are looking for. \\ {\bf Step 2.} We define
\(\conc(A, b, \pi)\) to be the set of all the constraints on \(I\) and \(\pi\) given
by the initialization and induction constraints
(\cref{item-inv-begin,item-inv-induct}) above, along with the following set of
consistency constraints for \(\pi\): for all states \(s\) of \(M\),
\(\pi(s, a) \geq 0\) for all actions \(a \in \act(s)\), and
    \(\sum_{a \in \act(s)}\pi(s, a) = 1.\)\\
{\bf Step 3.} The constraints arising in \(\conc\inp{A, b, \pi}\) from the
inductive constraint (\cref{item-inv-induct}) involve a universal quantification
(\cref{eq-impl-linear}), which by the use of Farkas' lemma, can be converted
into an equivalent feasibility question about a system of inequalities (with
additional variables).  We carry out this transformation for each universally quantified constraint of the form given by \cref{eq-impl-linear},
and henceforth describe the equivalent form of \(\conc\inp{A, b, \pi}\) so obtained
also by the same notation.  After this sequence of transformations,
\(\conc\inp{A, b, \pi}\) is a conjunction of polynomial inequalities in the variables
\(A\) , \(b\), \(\pi\) (which appear as free variables) and a collection of
auxiliary variables (from the translation described above), so that
all auxiliary variables are existentially quantified.\\
{\bf Step 4.} With this setup, the problem of finding the minimum threshold \(\gamma\) for which
\(\predicate_{\gamma}\) can be forced can be encoded as
\begin{equation}
  \label{eq-transformed-minmax}
  \inf_{A, b, \pi : \conc(A, b, \pi) \text{ holds}}
  \inb{
    \begin{array}{rl}
      \sup\limits_{x \st x \geq 0} & H(x)\\
      \text{subj.~to} & Ax \leq b,
                         \vec{1}^{\top}x = 1
    \end{array}
  }.
\end{equation}
Again, the min-max and non-linear nature of this problem is not easily handled
by standard solvers, so we attempt to convert the problem into a single
optimization problem, via the use of duality for maximum entropy
(\cref{prop-entropy-dual}).  We note that when \(A, b, \pi\) satisfy
\(\conc\inp{A, b, \pi}\), the feasibility constraint required in
\cref{prop-entropy-dual} holds (due to \cref{item-inv-begin}).  Thus, the
problem above is equivalent to
\begin{equation}
  \label{eq-transformed-final}
  \inf\limits_{A, b, \pi : \conc(A, b, \pi) \text{ holds}}
  \inf\limits_{\lambda \geq 0} g(\lambda),
\end{equation}
where
\(g(\lambda) \defeq b^{\top}\lambda + \log \inp{\sum_{i=1}^n \exp\inp{-\lambda^{\top}A_i}},\) with
\(A_i\) denoting the \(i\)th column of \(A\), is a convex function of
\(\lambda\). Thus, \textsc{AlgoENT} resulting from the above four steps satisfies:
\begin{theorem}
  Given MDP $M$, initial distribution $\mu_0$, and warm-up parameter
  $K$,\begin{itemize}
    \item (soundness) if \textsc{AlgoENT} returns  $\pi$, invariant $I$, threshold $\gamma$, then $\forall t\geq K$, $\predicate_\gamma(M^\pi(\mu_0,t))$ holds.
    \item (conditional relative completeness) for a rational threshold $\gamma$, and
      a positive integer $m$, if there exists a memoryless strategy $\pi$ and an
      affine invariant $I$ of template size $m$ which satisfies the three
      conditions (items 1 to 3) with respect to $\predicate_{\gamma'}$ for some
      $\gamma' <\gamma$ then \textsc{AlgoENT} will detect this, assuming decidability of
      the theory of reals with exponentiation.
\end{itemize}
\end{theorem}
\begin{proof}
  \textbf{Soundness.} Suppose that we obtain values for \(A, b, \pi\) and
  \(\lambda\) which satisfy the constraints in the program in
  \cref{eq-transformed-final} and for which \(g(\lambda) \leq \gamma\).  Then, it follows from
  the equivalence of \cref{eq-transformed-final} and
  \cref{eq-transformed-minmax} that the invariant \(I\) described by \(A\) and
  \(b\), in addition to satisfying the initialization and inductive constraints
  (\cref{item-inv-begin,item-inv-induct}), also implies the following version of
  the safety constraint (\cref{item-inv-safe}):
\(\forall x, (I(x) \implies H(x) \leq \gamma).\)
By the discussion at the beginning of this section, this therefore implies that
\(\pi\) satisfies the requirement that \(M^{\pi}(\mu_0, t)\) has entropy at most \(\gamma\) at all times
\(t \geq K\).\\
\textbf{(Conditional) Relative Completeness.} Suppose that we are given a
rational threshold \(\gamma\), and also that for some positive integer \(m\), there
exists an affine invariant \(I\) of template size \(m\) which satisfies the
three conditions (\cref{item-inv-begin,item-inv-induct,item-inv-safe}) with
respect to \(\predicate_{\gamma'}\) for some \(\gamma' < \gamma\).  In order to prove relative
completeness, we need to give a procedure that detects that such an invariant
exists.  We now show that assuming a conjecture from transcendental number
theory, such a procedure exists.
We notice that the formulation in \cref{eq-transformed-final} implies
that such an invariant exists if and only if the optimization problem defined in
\cref{eq-transformed-final} with template size \(m\) has optimal value less than
\(\gamma\).   Given the form of the optimization problem, this occurs if and only if
there exist \(A, b, \pi\) and \(\lambda\) satisfying
\begin{equation}
 \label{eq-etr-exp-sentence}
 \begin{gathered}
   \lambda \geq 0;  \conc(A, b, \pi) \text{ holds}, \text { and}\\
   \sum_{i=1}^{n}\exp\inp{-\lambda^{\top}A_i} < \exp(\gamma - b^T\lambda).
 \end{gathered}
\end{equation}
Recall that \(\conc(A, b, \pi)\) is a system of polynomial inequalities. The
sentence in \cref{eq-etr-exp-sentence} is, therefore, a sentence in the first
order language of real closed fields augmented with the exponential function.
Unfortunately, in contrast to the case of the existential theory of reals, which
is known to be decidable in PSPACE~\cite{Canny88}, we are not aware of any
unconditional decidability results for this augmented theory. However, it was
shown by \citet{MacinW96} that assuming a famous conjecture in transcendental
number theory known as the real version of Schanuel's conjecture, the first
order theory of reals augmented with the exponential function is decidable (see
\citet{Wilki97} for an exposition of this result and the underlying conjecture).
Since the existential theory of reals with the exponential function is a
fragment of this theory, we obtain that the sentence in
\cref{eq-etr-exp-sentence} is also decidable assuming the same conjecture.
\end{proof}

We describe two more salient features of  \textsc{AlgoEnt}. The first is regarding the complexity of verifying whether the answers reported by the algorithm are indeed correct.

{\bf Verification.} Suppose we are given some arbitrary values of \(A, b\)
and \(\pi\), and it is claimed that these describe an invariant \(I\) and a
strategy \(\pi\) for which the initialization, inductive and safety constraints
(\cref{item-inv-begin,item-inv-induct,item-inv-safe}) are all satisfied, for a
certain threshold \(\gamma\).  We can show that this claim can be verified in polynomial time. In other words, verifying the answer given out by \textsc{AlgoEnt} is easy, up to an additive approximation factor. Formally, 

\begin{proposition}[\textbf{Approximately verifying a candidate invariant}]\label{prop:verif}
  Given the MDP \(M\), initial distribution \(\mu_{0}\), {warm-up parameter \(K\), } a candidate description $(A,b)$ of the invariant $I$ (using
  \cref{eq-invariant}), a strategy $\pi$, a threshold $\gamma$ and additive
  approximation parameter $\epsilon$, all specified in terms of rational numbers
  represented as ratios of binary integers, we can check in time polynomial in
  \(K\) and the representation lengths of the input whether the following conditions
  hold:
  \begin{enumerate}
  \item \textbf{(Initialization)} \(I(M^{\pi}(\mu_0, K))\)
    holds;
  \item \textbf{(Induction)} \(\forall x, I(x) \implies I(M^{\pi}(x,
    1))\) ; and
  \item \textbf{(Invariant implies safety)}
    \(\forall x, I({x}) \implies (H(x) \leq \gamma + \epsilon)\).
  \end{enumerate}
\end{proposition}
\begin{proof} We first recall that for \emph{fixed} \(A, b \) and \(\pi\),
  \(\conc(A, b, \pi)\) is a feasibility question for a system of (non-strict)
  linear equalities in the auxiliary variables.  {Thus, whether \(\conc\inp{A, b, \pi}\) is satisfied can be checked by a linear
    program solver in time polynomial in \(K\) and the representation length of
    the input: the polynomial dependence on \(K\) arises because the
    representation length of the distribution \(M^{\pi}(\mu_0, K)\) can be at most
    polynomial in \(K\). }
It remains to check whether the inner maximization problem in
  \cref{eq-transformed-minmax} has optimal value \(\gamma\).  To do this, we note
  that this problem involves the maximization of the \emph{concave} entropy
  function over the (compact) probability simplex, with a further set of linear
  constraints.  Under these conditions, a \(\pm\epsilon\) additive approximation to the
  optimal value of this problem can be obtained in time that is polynomial in
  \(\log(1/\epsilon)\), via a standard result of \citet[Theorem 4.3.13]{GrotsLS93},
  based on the shallow cut ellipsoid method of \citet{YN76} {(see also Section 13.3 of \cite{vishnoiAlgorithmsConvexOptimization2021} for
    an explicit statement). } Thus, this procedure lets us certify whether the claim holds with the slightly
  relaxed threshold \(\gamma + \epsilon\) (in time that grows polynomially in
  \(\log(1/\epsilon)\)).
\end{proof}

{\bf Numerical Viability.} Finally, we address the implementability of the above procedure. The fact that for fixed \(A, b\) and \(\pi\), the
underlying optimization problem is convex suggests the use of non-linear
solvers (such as Gurobi~\cite{gurobi}) that are optimized for convex problems.  These solvers
typically require that the problem should be in the following form: optimize a
function subject to a system of inequalities.  The crucial point is that while the min-max formulation in
\cref{eq-transformed-minmax} is not of such a form, the final formulation that we
obtain in \cref{eq-transformed-final} is.  This allows us to use state-of-the-art solvers on our final formulation. In Section~\ref{sec:experiments}, we
report experimental results testing out the approach of feeding the problem in
\cref{eq-transformed-final} to such a solver.

\subsection{A sufficient condition for affine invariants}
We now demonstrate a natural situation in which affine invariants indeed suffice
to certify a low-entropy condition (despite the non-linearity of the latter).
Recall that a Markov chain is said to be \emph{ergodic} if it is irreducible and
aperiodic, and in that case it has a unique stationary distribution to which it
converges~(see, e.g., \cite[Theorem 4.9]{LP2017}). We show that if the MDP \(M\)
admits a memoryless policy \(\pi\) for which
\begin{enumerate}
\item the induced Markov chain \(M^{\pi}\) is ergodic, and
\item the entropy \(H(\mu^{*})\) of the unique stationary distribution
  \(\mu^{*}\) of \(M^{\pi}\) is below the goal threshold, i.e., \(H(\mu^{*}) < \gamma\),
\end{enumerate}
then there exists a warm-up parameter \(K\) and an affine invariant \(I\) of
size polynomial in the number of states of \(M\) that certifies the low-entropy
condition for \(\pi\), i.e., that \(H(M^{\pi}(\mu_0, t)) < \gamma\) for all
\(t \geq K\). The intuition is that under the above conditions, the MDP must
converge to the stationary distribution, so that after a certain number of steps
(depending on the mixing time of the induced Markov chain), we will get within a
small ball around the stationary distribution. Thus, using the continuity of the
entropy, we can find a convex linear invariant around this stationary
distribution that witnesses the safety condition. One subtlety in this analysis is that
the linear invariant capturing such a ball does not \emph{a priori} need to be
of small size.  For example, the natural way to define such a ball is in terms
of the total-variation (or, equivalently, \(\ell_1\)) norm, in which case the
resulting invariant would be of size exponential in the number of states in the
MDP.  However, by looking at a ball defined according to a carefully chosen
weighted version of the $\ell_\infty$-norm we can ensure that the invariant is of size
at most polynomial in the size of the MDP. We formalize this argument below.
\begin{proposition}
  Let $M$ be an MDP on a finite state space \(S\) and let $\pi$ be a memoryless
  policy such that the induced Markov chain $M^\pi$ is
  ergodic. Let $\mu^* \in \Delta(S)$ be the unique stationary distribution of \(M^{\pi}\).
Further, assume that the safe set is of the form $H(x) \le \gamma$, for a
  \(\gamma\) satisfying $H(\mu^*) < \gamma$. Then for any starting distribution \(\mu_0\), there exists a warm-up parameter
  $K \in \mathbb{N}$ and an affine invariant $I$ defined by \(O(|S|)\) linear inequalities
  such that:
  \begin{enumerate}
  \item \textbf{Initialization:} \(M^{\pi}(\mu_0, K) \in I\).
  \item \textbf{Induction:} \(\forall x\, I(x) \implies I(M^\pi(x, 1))\).
  \item \textbf{Safety:} \(\forall x\, I(x) \implies H(x) \le \gamma\).
\end{enumerate}
Thus, there exists an affine invariant of the form sought for by our algorithm. \end{proposition}

\begin{proof}

We first construct a neighborhood around the stationary distribution where the
  entropy is guaranteed to be at most $\gamma$. To do so, we start by observing that
  the entropy function $H$ is continuous on the simplex $\Delta(S)$. Since
  $H(\mu^*) < \gamma$, let $\varepsilon := \gamma - H(\mu^*) > 0.  $ Now, we define
  $ B_\delta := \{x \in \Delta(S) \mid \|x - \mu^*\|_\infty \le \delta\}, $ for a particular
  $\delta > 0$ that we will soon determine. We note also that since \(M^{\pi}\) is ergodic, all entries of \(\mu^{*}\) are
  strictly positive.  One constraint on our eventual choice of \(\delta\) is that
  \(\delta\) has to be smaller than the smallest entry of \(\mu^{*}\): this constraint
  on \(\delta\) ensures that all \(x \in B_{\delta}\) have positive entries.

  Let \(n = |S|\) be the number of states, and for ease of notation, we index
  the states as \(\inb{1, 2, \dots, n}\).  We define
  $f: \R_{\geq 0} \rightarrow \R$ by \(f(x)=-x\log{x}\) (with
  \(f(0) \defeq 0\)), so that $H(x)-H(\mu^*)=\sum_{i=1}^n (f(x_i)-f(\mu^*_i))$.
  Further, \(f\) is a concave function, so that we have
  \(f(b) - f(a) \leq f'(a)(b-a)\) for all \(a, b > 0\).  We thus have, for each
  coordinate \(i\),
  \begin{equation}
    \label{eq:10}
    f(x_i) - f(\mu_i^{*}) \leq -(1+\log \mu_i^{*})\cdot(x_i - \mu_i^{*}).
  \end{equation}
Now, when \(x \in B_{\delta}\), we have $|x_i - \mu_i^{*}| \le \delta$ for all
  $i$.  Using this in \cref{eq:10}, we thus get
  \begin{equation}
    \label{eq:11}
    H(x) - H(\mu^{*}) \le \delta \sum_{i=1}^n \abs{\log \mu_i^{*} + 1}. \
  \end{equation}
Thus, by choosing
$\delta < \min\inb{\frac{\epsilon}{\sum_{i=1}^n |\log \mu_i^{*} + 1|}, \min_i\mu_{i}^{*}}$, we get that all
\(x \in B_{\delta}\) have positive entries, and also that
\begin{align}\sup_{x\in B_\delta} H(x) \leq H(\mu^{*}) + \epsilon \leq   \gamma
\label{eq:nbd}
\end{align}

We now have to show that there exists an affine invariant \(I\) of small size
such that \(I(x)\) implies that \(x \in B_{\delta}\).  We claim that the following
affine invariant \(I(x)\) works:
\begin{equation}
  I(x) \defeq \left\{\begin{gathered}
    x_i \geq 0 \text{ for all } 1 \leq i \leq n,\text{ and }\\
    \sum_{i=1}^nx_i = 1,\text{ and }\\
    -\delta \cdot \mu_i^{*}\leq x_i - \mu_i^{*}\leq \delta \cdot \mu_i^{*} \text{ for all } 1 \leq i \leq n. \label{eq:13}
  \end{gathered} \right.
\end{equation}
More succinctly, the last condition in the affine invariant in \cref{eq:13} can
be written in the form of a modified \(\ell_{\infty}\) norm inequality as
\begin{equation}
  \label{eq:15}
  \norm[\infty]{(x - \mu^{*})\diag(\mu^{*})^{-1}} \leq \delta,
\end{equation}
where \(\diag(\mu^{*})\) denotes the \(n \times n\) diagonal matrix with the entries of
\(\mu^{*}\), in order, on the diagonal.  Further, if \(I(x)\) is satisfied, then
we also have \(x \in B_{\delta}\), since the \(\mu_i^{*}\) are all at most \(1\).

Note that the affine invariant \(I\) defined in \cref{eq:13} has size at most
\(O(n)\).  We now verify that \(I\) satisfies the three required conditions.

\noindent \textbf{Safety:} As noted above, if \(I(x)\) is satisfied, then
\(x \in B_{\delta}\).  Thus, \cref{eq:nbd} shows that \(I(x)\) implies the safety
condition \(H(x) \leq \gamma\). \noindent \textbf{Induction:} Suppose that \(I(x)\) is satisfied. For ease of
notation let \(M\) denote the transition matrix of the Markov chain \(M^{\pi}\),
and let \(y \defeq M^{\pi}(x, 1) = x M\). Note that since \(x\) is a probability
vector and \(M\) is the transition matrix of a Markov chain, the first two
conditions in the definition of \(I(y)\) (\cref{eq:13}) are automatically
satisfied. We now check the third condition. Define
\(v \defeq y - \mu^{*} = (x - \mu^{*})M\). We then have, for each \(1 \leq i \leq n\),
\begin{align}
  \abs{y_i - \mu_i^{*}} = \abs{v_{i}}
  &= \abs*{\sum_{j = 1}^{n}(x_j - \mu_j^{*})\cdot M_{ji}} 
    \leq \sum_{j = 1}^{n}\abs{x_j - \mu_j^{*}}\cdot M_{ji}\\
  &\leq \delta \sum_{j = 1}^{n}\mu_j^{*}\cdot M_{ji} = \delta\cdot \mu_i^{*},\label{eq:16}
\end{align}
where the inequality in \cref{eq:16} holds because of the last condition in the
specification of \(I(x)\) (\cref{eq:13}).  This verifies that if \(I(x)\) holds,
then so does \(I(M^{\pi}(x, 1))\).

\noindent \textbf{Initialization:} To establish this, we need to show that there
exists a \(K\) so that \(I(M^{\pi}(\mu_{0}, K))\) holds. Choose \(K\) to be the
first time at which
\(\norm[1]{M^{\pi}(\mu_{0}, K) - \mu^{*}} \leq \delta\cdot\min_i\mu_i^{*}\). Since
\(M^{\pi}\) is assumed to be ergodic, such a \(K\) exists. Further, for this
\(K\), \(I(M^{\pi}(\mu_{0}, K))\) holds since
\(v \defeq M^{\pi}(\mu_{0}, K) - \mu^{*}\) satisfies that for each
\(1 \leq i \leq n\),
\[\abs{v_i} \leq \norm[1]{M^{\pi}(\mu_{0}, K) - \mu^{*}} \leq \delta \mu_i^{*}.\qedhere\]
\end{proof}
We remark that if a a lower bound on the rate of convergence of \(M^{\pi}\) is
known, we can use this information to compute an upper bound on the warm-up
parameter \(K\) required in the above proposition, thus removing the dependence
on $K$ in the proposition above.

\section{Entropy, Memory and Randomization}
\label{sec:random-memory}
In this paper we have focused on synthesizing memoryless strategies. An immediate question from a theoretical point of view is whether such strategies always suffice. For reachability and safety objectives that are {\em state-based}, it is well-known
that deterministic and memoryless strategies
suffice~\cite[Section 4.4]{Puterman-book}. 
In contrast, we show next that for our entropy objectives, memoryless strategies
do not suffice, and even strategies that merely keep track of the current time
step can do better. As mentioned in the introduction, this also differentiates our work from path-based entropy~\cite{savas2019entropy,BiondLNW14,DBLP:conf/fsttcs/ChenH14}.

\begin{proposition}
\label{prop:no-state-yes-dist}
There exists an instance of \initunfairDec{} for which there are no memoryless strategies that answer YES, but there exists a memoryful strategy
(which depends only on the parity of the current time) that answers YES. \end{proposition}

\begin{proof}
 Consider the following MDP $M_3$ in Figure~\ref{fig:MDPM3}, with $K = 1$, and $\mu_0$ being the distribution that gives probability mass 1 in the top state and 0 elsewhere. Consider the case in which we are only restricted to (randomized) memoryless strategies. To parametrize those strategies, let the probability of choosing $a$ be $x$.
 \begin{figure}[h]
\begin{center}
\begin{tikzpicture}[scale=0.1]
    \tikzstyle{every node}+=[inner sep=0pt]
    \draw [black] (34.5,-14.9) circle (3);
    \draw [black] (24.2,-31.5) circle (3);
    \draw [black] (11.9,-31.5) circle (3);
    \draw [black] (44.3,-31.5) circle (3);
    \draw [black] (56.7,-31.5) circle (3);
    \draw [black] (32.92,-17.45) -- (25.78,-28.95);
    \fill [black] (25.78,-28.95) -- (26.63,-28.53) -- (25.78,-28.01);
    \draw (28.72,-21.91) node [left] {$a,0.7$};
    \draw [black] (13.055,-28.763) arc (143.32255:36.67745:6.228);
    \fill [black] (13.06,-28.76) -- (13.93,-28.42) -- (13.13,-27.82);
    \draw (18.05,-25.75) node [above] {$1$};
    \draw [black] (23.089,-34.255) arc (-35.79607:-144.20393:6.212);
    \fill [black] (23.09,-34.26) -- (22.22,-34.61) -- (23.03,-35.2);
    \draw (18.05,-37.33) node [below] {$1$};
    \draw [black] (45.391,-28.736) arc (144.72676:35.27324:6.258);
    \fill [black] (55.61,-28.74) -- (55.56,-27.79) -- (54.74,-28.37);
    \draw (50.5,-25.59) node [above] {$0.5$};
    \draw [black] (59.38,-30.177) arc (144:-144:2.25);
    \draw (63.95,-31.5) node [right] {$0.5$};
    \fill [black] (59.38,-32.82) -- (59.73,-33.7) -- (60.32,-32.89);
    \draw [black] (55.21,-34.073) arc (-43.35523:-136.64477:6.477);
    \fill [black] (45.79,-34.07) -- (45.98,-35) -- (46.7,-34.31);
    \draw (50.5,-36.6) node [below] {$0.5$};
    \draw [black] (41.477,-32.48) arc (316.87498:28.87498:2.25);
    \draw (36.93,-30.53) node [left] {$0.5$};
    \fill [black] (41.81,-29.86) -- (41.56,-28.94) -- (40.88,-29.67);
    \draw [black] (36.03,-17.48) -- (42.77,-28.92);
    \fill [black] (42.77,-28.92) -- (42.8,-27.97) -- (41.94,-28.48);
    \draw (48.75,-21.45) node [left] {$b,0.05$};
    \draw [black] (37.466,-14.473) arc (93.45714:12.96842:18.172);
    \fill [black] (56.27,-28.53) -- (56.58,-27.64) -- (55.6,-27.87);
    \draw (52.45,-17.56) node [above] {$b,0.05$};
    \draw [black] (31.689,-13.886) arc (277.90143:-10.09857:2.25);
    \draw (27.3,-9.55) node [above] {$b,0.9$};
    \fill [black] (33.59,-12.05) -- (33.98,-11.19) -- (32.99,-11.33);
    \draw [black] (35.359,-12.038) arc (191.02106:-96.97894:2.25);
    \draw (41.56,-9.48) node [above] {$a,0.3$};
    \fill [black] (37.29,-13.84) -- (38.18,-14.18) -- (37.98,-13.2);
\end{tikzpicture}
\label{fig:counterexample-1}
\end{center}
\caption{MDP $M_3$}
\label{fig:MDPM3}
\end{figure}

We then plot the entropy at times \(t = 1, 2, 3\) as
a function of \(x\).  The corresponding entropies are given below.
\\After 1 step:
\begin{equation*}
  -\left(0.9-0.6x\right)\ln\left(\left(0.9-0.6x\right)\right) -\left(0.1\left(1-x\right)\right)\ln\left(\left(0.05\left(1-x\right)\right)\right)
  -0.7x\ln\left(0.7x\right).
\end{equation*}
After 2 steps:
\begin{multline*}
  -2\left(0.9-0.6x\right)^{2}\ln\left(0.9-0.6x\right)\\
- \big(0.1\left(1-x\right)\left(1.9-0.6x\right)\big) \cdot\ln\left(0.05\left(1-x\right)\left(1.9-0.6x\right)\right)
-\left(0.7x\right)\ln\left(0.7x\right)\\
-0.7x\left(0.9-0.6x\right)\ln\left(0.7x\left(0.9-0.6x\right)\right).
\end{multline*}
After 3 steps:
\begin{multline*}
  -3\left(0.9-0.6x\right)^{3}\ln\left(0.9-0.6x\right)\\
  -0.7x\left(0.9-0.6x\right)\ln\left(0.7x\left(0.9-0.6x\right)\right) -0.7x\left(1+\left(0.9-0.6x\right)^{2}\right)\ln\left(0.7x\left(1+\left(0.9-0.6x\right)^{2}\right)\right)\\
  -\left(0.1\left(1-x\right)\left(1+0.9-0.6x+\left(0.9-0.6x\right)^{2}\right)\right)
\cdot \ln\left(0.05\left(1-x\right)
    \left(1+0.9-0.6x+\left(0.9-0.6x\right)^{2}\right)\right).
\end{multline*}
\noindent From these plots, we see that any such strategy will have a maximum
entropy of at least 0.77 (to observe this, we note that for every \(x\), i.e.,
for every such strategy, at least one of the entropy functions is at least 0.77;
see \cref{fig-prop10}). On the other hand, the non-memoryless strategy (which
just keeps track just of the parity of the current time using just one bit of
memory!) of alternating between choosing $a$ and $b$ obtains a maximum entropy
of at most 0.73. Thus this example illustrates that memoryless strategies are
not enough to reach an optimal value.
\begin{figure}[ht!]
  \centering
  \includegraphics[width=0.45\textwidth]{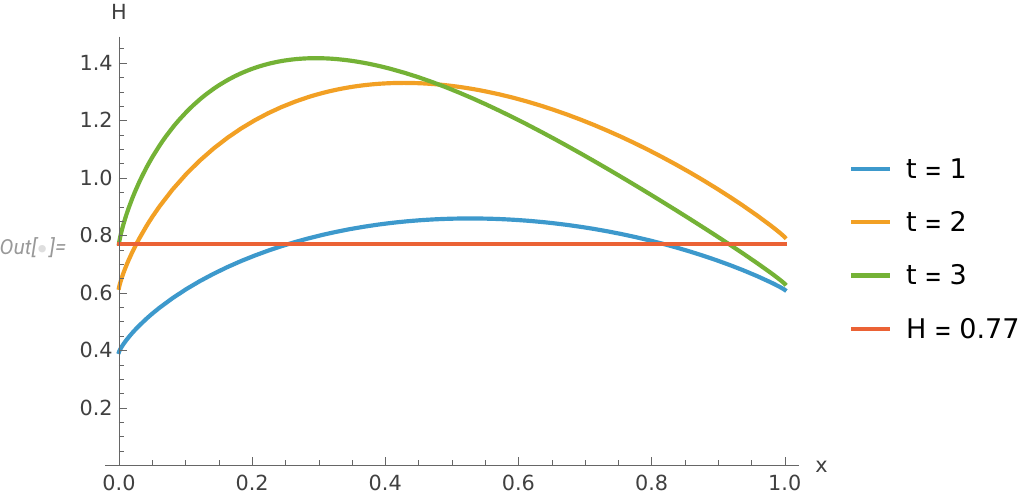}
  \caption{Entropy functions in the proof of Prop.~\ref{prop:no-state-yes-dist}}
  \label{fig-prop10}
\end{figure}
\end{proof}

This leads us to the question of what memory is needed for ensuring entropy objectives. An approach towards this is provided in \cite{akshay2023mdps}, where  it was shown that, for distributional {\em affine safety objectives}, while memoryless strategies do \emph{not} suffice, {\em distributional memoryless strategies} do.
Informally, the only memory {\em distributional memoryless} strategies carry is of the current \emph{distribution} of the
MDP.  Formally, distributional memoryless strategies are defined as functions from the set \(\Delta(\Omega)\) of probability distributions over states to the set \memoryless{} of memoryless strategies. 
We extend this result to entropy objectives, by showing that distributional memoryless strategies suffice for entropy objectives as well. 

\begin{proposition}
\label{prop:memory}
 For every instance of \initunfairDec{}, there exists a (possibly infinite
  memory) strategy that answers YES iff there exists a distributional memoryless
  strategy that answers YES.
\end{proposition}

\begin{proof}
  The proof essentially follows from the fact that the \initunfairDec{} problem
  is a special case of the distributional safety problem for which this
  characterization was shown in ~\cite{akshay2023mdps}. For the sake of
  completeness, we provide the proof here.  For the reverse direction,
  existence of a distributional memoryless policy that answers YES is evidence
  of a policy that answers YES. For the forward direction, let us assume a
  strategy (possibly with infinite memory) exists, and we observe
  the probability distributions we get if we follow this strategy starting from
  $\mu_0$. We now construct a distributional memoryless strategy using this
  strategy. Until such time as we see a probability distribution repeat, our
  distributional memoryless strategy will replicate the probability
  distribution over actions that the original strategy takes given that
  particular distribution. If we ever observe the same probability
    distribution over states twice, the entropies of all distributions arising in
    between these two observations must have all been at most $\gamma$, and our
    distributional memoryless strategy can keep looping over this sequence of
    states deterministically.
\end{proof}

{\bf Randomization.} Finally, we turn to the question of randomized vs
deterministic strategies. Our results so far, and especially in Section~\ref{sec:algo}, considered randomized strategies by default. But one could ask if randomization adds expressive power to memoryless strategies. Interestingly, we show that for entropy objectives, this is the case regardless of whether we consider usual (i.e., state-based) memoryless strategies or distributional memoryless strategies. 
\begin{proposition}
  There exists an instance of \initunfairDec{} for which there exists no deterministic
  memoryless strategy that answers YES, but for which there exists a randomized
  memoryless strategy that answers YES.
\end{proposition}

\begin{proof}
Consider the MDP sketched in \cref{fig:MDP-rand}, with $K=1$ and the initial distribution having all
  its probability mass on node \(A\). All unlabeled transitions have probability 1. Any  deterministic memoryless strategy here will have to choose either action $a$ or $b$ at the first step. And after this, regardless of the choice, the entropy would be $\ln{5}$ at timestep 1 if action $a$ is picked or at timestep 2 if action $b$ is picked. On the other
  hand, consider a randomized memoryless strategy that chooses $a$ or $b$ with probability
  \(1/2\) each.  When following this strategy, the entropy from the first
  timestep onwards will be $\frac{\ln{5}}{2} + \ln{2}$. In general, for a
  strategy that chooses action \(a\) with probability \(p\), the maximum entropy
  of the system is $\max\{p, 1-p\}\cdot\ln{5}-p\ln{p}-(1-p)\ln{(1-p)}$, which is
  easily verified to have a minimum at $p=1/2$. Thus, the optimal strategy in
  this case is a randomized strategy, which performs better than any
  deterministic strategy.
  \begin{figure}[t]
\begin{center}
  \begin{tikzpicture}[>=stealth, node distance=0.7cm and 0.5cm,scale=0.15]
    \node[state] (A) {\(A\)};

    \node[state] (B) [right=of A] {};

    \node[state] (B1) [below=of B] {};

    \node[state] (A1) [below=of A] {};

    \node[state] (AS) [below=of A1] {};

    \path [->] (A) edge node [left, near end] {\(a, 1/5\)} (A1)edge node [above] {\(b, 1\)} (B) (A1) edge (AS) (AS) edge [loop below] () (B) edge node[right, near end] {\(1/5\)} (B1) (B1) edge [loop below] ();

    \foreach \x in {2,...,5}{\pgfmathtruncatemacro{\prev}{\x - 1}; 

      \node[state] (A\x) [left=of A\prev] {};      
      \node[state] (B\x) [right=of B\prev] {};      

      \path [->] (A) edge [bend right=20] node [left, near end] {\(a, 1/5\)} (A\x) (A\x) edge [bend right=20] (AS) (B) edge [bend left=20] node[right, near end] {\(1/5\)} (B\x) (B\x) edge [loop below] ();

    }
  \end{tikzpicture}
\end{center}
\caption{Example exhibiting the need for randomized strategies}
\label{fig:MDP-rand}

\end{figure}
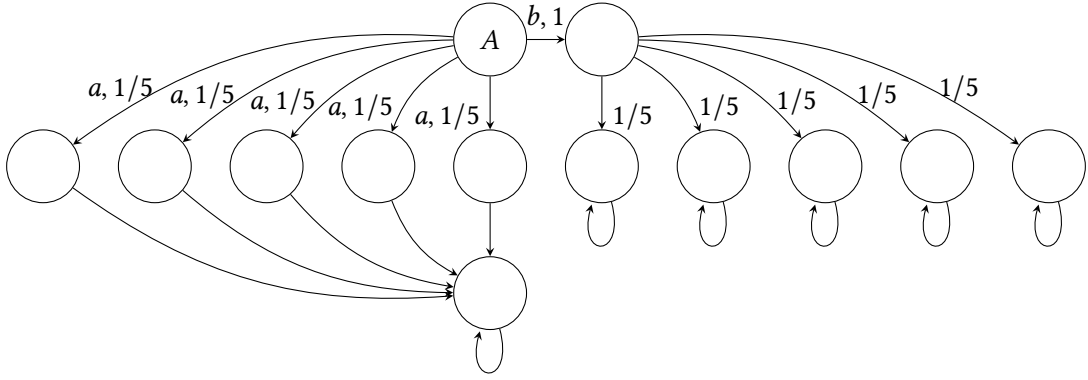
\end{proof}

In fact, one can observe that the example in the proof above also shows that even for distributional strategies, randomization may be needed. 

 \begin{corollary}
 \label{prop:dist-rand}
 There exists an instance of \initunfairDec{} for which there is no deterministic distributional memoryless strategy that answers YES, but for which there exists a (randomized) distributional memoryless strategy that answers YES.
 \end{corollary} 
{ The proof follows by observing that in the example in Figure~\ref{fig:MDP-rand}, any deterministic distributional strategy will also have to choose action $a$ or $b$ resulting in the same entropy. The rest of the proof follows as before by additionally observing that any usual, i.e., ``state-based'' strategy is also a distributional strategy.
 }

{In summary, this section shows that distributional memoryless strategies
  are necessary and sufficient for ensuring low entropy objectives, and are
  strictly more powerful than usual (state-based) memoryless
  strategies. Moreover randomization adds power for the entropy objective.

\section{Experiments}
\label{sec:experiments}
While the focus of this work is on a theoretical exploration of entropy
objectives, we also implemented a prototype to show the promise and potential of
our approach. Our implementation is in Python3, builds upon the implementation
framework developed in \cite{akshay2023mdps}, and was run on an AMD Ryzen 7530U
CPU with 16GB RAM, using Gurobi~\cite{gurobi}, a commercial solver with an
academic license. We employ \texttt{sympy}~\cite{sympy} for the templates and
\texttt{pySMT}~\cite{pysmt} (with the \texttt{mathsat} backend~\cite{mathsat5})
for verification of the synthesized invariants. For the sake of efficiency, Gurobi was set to have a tolerance of 1e-9, an
iteration limit of 200000, and its search space for all real number variables
was bounded between -5000 and +5000 (though for some examples, we had to do away
with some of this restrictions). As mentioned in the Introduction, we are not
aware of any existing tools that solve the entropy objective. Hence, we are
unable to perform a reasonable comparison with existing work. In this section,
our focus is to show that we can solve interesting examples and MDPs of a
reasonable size.

Given an MDP and the number of invariant constraints, we employ \texttt{sympy} to encode all our constraints. A Gurobi model is then initialized with all required template variables, the constraints and objective function. Its output is fed to \texttt{pySMT} to verify the correctness of the synthesized invariant constraints under the generated policy (which is memoryless in our template).

\medskip 

\noindent\textbf{Benchmarks.}
We evaluate our approach on several synthetic and existing benchmarks with
respect to an entropy objective. Rather than checking if we are lower than a
threshold $\gamma$ for evaluation, we try to compute the minimal value that can be
obtained. As a result, the value computed by our approach will always be an
upper bound on the real minimum. Some of the benchmarks were Markov chains where
the objective was to synthesize invariants and compute the optimal low entropy. For MDPs, we also synthesize the strategy that achieves this entropy value. 

We present results on seven MDP benchmarks, where we could manually compute the
correct entropy computation and check it against our prototype's result. Of
these $M_1,M_2, M_3$ are respectively the MDPs considered earlier in this paper
in Example~\ref{eg:MDPM1}, Example~\ref{eg:MDPM2}, and Figure~\ref{fig:MDPM3}.
$M_4$ (in Appendix~\ref{app:benchmarks}) is another synthetic example with five
states and two actions, where we again see that randomization is needed for
optimality. Beyond this, we considered the MDP \textbf{Split} taken
from~\cite{akshay2023mdps}. We also considered two gridworld environments
\textbf{Grid1}, which is a grid of size 2x4 with the point (0,1) being
unreachable and $\textbf{Grid2}$ of size 3x4 with 3 unreachable states. In
addition, we also present results on examples of Markov chains, including both
synthetics ones and some taken from the
literature~\cite{DBLP:journals/jacm/AgrawalAGT15}. More details
regarding the benchmarks as well as the initial distributions, and a few more
benchmarks are in Appendix~\ref{app:benchmarks}. We vary the warm-up parameter
$K$ and show results for different values of $K$. For all these benchmarks, one
reason to choose them was to be able to interpret the results easily, since in
more complicated benchmarks, even if our tool does scale, we cannot check our
answers
easily. 

\medskip

\noindent\textbf{Results.}
We present a set of results in Table~\ref{tab:experiment_results},
with more details in Appendix~\ref{app:benchmarks}, including the invariants and memoryless strategies synthesized. Our prototype is able to solve all these tasks with different warm-up parameters. To be able to validate the answer computed by our approach (in column \textbf{EntAns}), we computed the optimal answer (in column \textbf{EntOpt}) in several small cases and in other cases, we ran simulations to heuristically observe a minimal answer. These answers match our approach often and we are always within a very small margin of the computed answers. The answers vary with varying $K$ since in these cases, entropy in initial steps was high and reduced as $K$ was increased.

\begin{table}[t]
\centering
\begin{tabular}{|l|c|c|c|c|}
\hline
\textbf{Name} & \textbf{\#States} & \textbf{K} & \textbf{EntAns} & \textbf{EntOpt} \\
\hline
\textbf{Type} & \textbf{MC} &   &   &   \\
\hline
MDP $M_1$ & 3 & 0 & 1.023 & 1.011 \\
\hline
MDP $M_1$ & 3 & 1 & 0.679 & 0.679 \\
\hline
MDP $M_1$ & 3 & 2 & 0.512 & 0.512 \\
\hline
MDP $M_2$ & 4 & 0 & 1.324 & 1.091 \\
\hline
MDP $M_3$ & 5 & 0  & 1.099 & 0.772 \\
\hline
MDP $M_4$ & 5 & 2 & 0.693 & 0.304 \\
\hline
\textbf{Type} & \textbf{MC} &   &   &   \\
\hline
MC1 & 2 & 1 & 0.637 & 0.637 \\
\hline
MC2 & 2 & 1 & 0.662 & 0.662 \\
\hline
Insulin & 5 & 0 & 1.001 & 0.864 \\
\hline
Pagerank & 5 & 0 & 1.537 & 1.443 \\
\hline
\end{tabular}

\caption{Comparison of Minimal Entropy Answer vs Optimal Answer across different experiments. \textbf{$K$} is the warm-up parameter. \textbf{EntAns} is the answer given by our prototype. \textbf{EntOpt} are optimal answers either hand-computed or derived by simulating the Markov Chain, or by grid search over all possible strategies of the MDP to identify the optimal strategy.  All values rounded to three places after decimal. }
\label{tab:experiment_results}
\end{table}

Our prototype is able to handle even larger benchmarks, but we are unable to validate the answers that we provide. However we do validate the invariants (which are intermediate steps) using an SMT-solver. Note that when we provide an answer, we have soundness guarantees subject to correctness of Gurobi, the back-end solver that we use.

\medskip

\noindent{\bf Discussion.} Two remarks are pertinent at this point. Our
templates for invariants are always convex and we can show that there are MDPs
for which the optimal entropy 
  cannot be achieved using templates that are affine or even convex functions. 
So in that sense our experiments always report an upper-bound on the minimal
entropy but may not reach the minimal. {This is sufficient to solve
  instances of \apxinitunfairDec{} problems in many cases, but does not
  necessarily obtain the minimum, often explaining the difference between our
  values and the minimum entropy in Table~\ref{tab:experiment_results}. Given
  the hardness of the problem, we believe this to be unavoidable. At the same
  time, we note that the (approximate) verification problem that we showed to be
  solvable in polynomial time in Proposition~\ref{prop:verif} can indeed be
  solved efficiently using tools such as Gurobi.}

{Second, our implementation currently depends on Gurobi~\cite{gurobi},
  which uses multiple techniques, heuristics and optimizations, but is not known
  to be complete even for theory of reals. Hence at times, it cannot handle even
  simple examples. This limitation of the solver also causes other interesting
  behaviors: for instance, 
the
 solver's performance can be highly dependent on the initial starting point, a
  phenomenon common to many other algorithms of this type.
}

\section{Conclusion}
\label{sec:conclusion} {In this paper, we launched an investigation into
  verification of distributional entropy objectives for MDPs. We considered the
  specific problem of synthesizing strategies to obtain low entropy in the
  distributions generated by an MDP. Thus in our setting, we check if there is a
  strategy that ensures entropy is below a threshold, thus enforcing a
  concentration like behavior in the state distributions of the MDP. Several
  related problems may also be equally interesting: for instance, one could ask
  if there is a strategy that enforces \emph{high} entropy (which would
  correspond to diversity or fairness), or whether all strategies after a point
  remain fair or concentrated (high or low entropy). It is not immediately clear
  whether the strong duality results that we utilize can be adapted to these
  settings, and addressing these different questions would form an interesting
  line of future research.}

{We showed NP-hardness (even for approximations), and developed a convex-duality and invariant-synthesis based algorithm to solve the problem that is always sound, and is relatively complete assuming Schanuel's conjecture. We leave decidability of the general low entropy decision problem open, but we expect it to be as hard as the so-called positivity problem~\cite{OuaknineW14}, similar to what was shown for affine distributional safety objectives~\cite{akshay2023mdps}. }

{As future work, we are interested in improving our templates in theory and in practice to be able to handle more benchmarks with stronger guarantees, and in further exploring the capability and limitations of convex invariants for distributional objectives.}

\section*{Acknowledgements}
We acknowledge support from the Department of Atomic Energy, Government of India
[project number RTI4014]; by the Infosys-Chandrasekharan virtual center for
Random Geometry at the Tata Institute of Fundamental Research; by the Science
and Engineering Research Board [grant number MATRICS MTR/2023/001547]; by the
Google India Research Award; and by a gift to TIFR from Adobe, Inc. The results
reported in this manuscript do not necessarily reflect the views of the funding
agencies listed above.

\newcommand{\etalchar}[1]{$^{#1}$}

\appendix
\onecolumn
\section{A strong duality result for entropy}
\label{app:proofs}

In this section, we provide the proof of Proposition~\ref{prop-entropy-dual}.  As mentioned in Section~\ref{sec:algo}, we need a strong duality result which is usually stated under the additional hypothesis that the feasible point \(x_0\) has all its entries {\em strictly} positive (this is the version given, e.g., by \citet{BV04}), in which form it is proved using Slater's
  constraint qualification.  However, as we show below, this additional hypothesis is not required for
    maximum-entropy programs with linear constraints, which is the setting that we are interested in. In this case, it is sufficient that the primal
    program has at least one feasible solution, and this can be shown for
    instance using the general duality results of \citet[Theorem
    30.4]{rockafellar70}. Again, we believe this result to be folklore, but for the
    sake of completeness, and due to its independent interest, we restate it here and provide a
    proof.

\addtocounter{theorem}{-9} 
\begin{proposition}
  Let \(Q \in \R^{\ell \times n}\) and \(p \in \R^\ell\) be fixed.  Let
  \(\D \defeq \inb{x | x \geq 0} \subseteq \R^n\) be the domain of the convex function
  \(f(x) \defeq \sum\limits_{i=1}^nx_i\log x_i\).  Suppose that there is an
  \(x_{0} \in \mathcal{D}\) satisfying \(Qx_0 \leq p\). Then the optimization program
  \begin{equation}
    \label{eq-ent-cons-app}
    \begin{aligned}
      \sup\limits_{x \in \D} H(x)
      \textup{ subject to }  Q x \leq p \textup{ and } \vec{1}^{\top}x = 1
    \end{aligned}
  \end{equation}
  has the same optimal value as the program
  \begin{equation}
    \label{eq-ent-dual-app}
    \inf\limits_{\lambda \geq 0} \quad g(\lambda) \defeq p^{\top}\lambda + \log \inp{\sum\limits_{i=1}^n\exp(-\lambda^{\top}Q_i)}.
  \end{equation}
  Here \(Q_i\) denotes \(i\)th column of \(Q\).  Further, \(g\) is a convex
  function of \(\lambda\).
\end{proposition}

\begin{proof}[Proof of Proposition~\ref{prop-entropy-dual}]
  One way to prove strong duality for the maximum-entropy program in
  \cref{eq-ent-cons-app}, without requiring a strictly positive feasible
  solution is to consider the so-called \emph{convex bifunction} \(F = F(u, x)\)
  associated with the program in
  \cref{eq-ent-cons-app}~\cite[pp.~293--294]{rockafellar70}. (Formally, to use
  the definition given by Rockafellar, one has to consider the equivalent
  problem of \emph{minimizing} the \emph{convex} function \(f(x) \defeq -H(x)\)
  under the same constraints as those in \cref{eq-ent-cons-app}.) Because of the
  continuity of \(H\) over \(\mathcal{D} = \inb{x \st x \geq 0}\), and the fact that all our
  constraints are affine, the resulting convex bifunction is, to use
  Rockafellar's terminology, both \emph{proper} (i.e. it is not identically
  equal to \(+\infty\) and takes values only in \(\R \cup \inb{\infty}\); in particular, it
  does not take the value \(-\infty\)) and \emph{closed} (in this context, this means
  that for any real \(w\), the set \(\inb{(u, x) | F(u, x) \leq w}\) is a closed
  set).

The two important features of the program in \cref{eq-ent-cons-app} for us are the
following:
\begin{itemize}
\item Given the assumption of primal feasibility, the feasible set of
  \cref{eq-ent-cons-app} is non-empty.  It is also a \emph{bounded} affine set,
  being an intersection of an affine set with the compact set
  \[\inb{x \st x \geq 0 \text { and } \vec{1}^{\top}x =1}.\]
\item For all \(x\) in this feasible set, we have \(f(x) = -H(x) \leq 0\).
\end{itemize}

In terms of Rockafellar's bifunction formalism, these two features translate
into the condition that the set \(\inb{x \st F(0, x) \leq 0}\) is a non-empty and
bounded set.  We now use Theorem 30.4 (g) combined with Theorem 30.3 of
  \citet{rockafellar70}, which together say that if the bifunction \(F\)
  associated with a convex program is closed and proper, and if there exists a
  real number \(\alpha\) such that the set \(\inb{x \st F(0, x) \leq \alpha}\) is non-empty,
  then strong duality (refereed to as \emph{normality} in \cite{rockafellar70})
  holds for the convex program.  In light of the discussion above about the
  \emph{closed} and \emph{proper} nature of the bifunction \(F\), the condition
  \(\inb{x \st F(0, x) \leq 0}\) therefore implies strong duality for our convex
  program. (Intuitively, the main technical idea underlying Rockafellar's strong duality
result cited above is the fact that for any \emph{proper} and \emph{closed}
convex function \(g\), the set of points where \(g\) achieves its minimum value
is a non-empty and bounded set if and only if the convex conjugate \(g^{*}\) of
\(g\) takes only finite values in some \emph{open} ball containing
\(0\)~\cite[Theorem 27.1 (d)]{rockafellar70}.)

Now, it only remains to note that the dual program described in Theorem 30.3 of
\citet{rockafellar70} has the same optimal value as \cref{eq-ent-dual-app}.  To see
this, we note that the derivation of the dual program described in Theorem 30.3
of \citet{rockafellar70} (and defined on pp.~309-311 of the same book), when
carried out for the max-entropy program of \cref{eq-ent-cons-app}, leads to exactly
the same computations as those in the derivation of \cref{eq-ent-dual-app} by
\citet[p.~228]{BV04}. The dual program thus turns out to have the same optimal
value as \cref{eq-ent-dual-app}.
\end{proof}

\section{Experiments}
In this section, we provide details regarding our experiments. As mentioned in the main text, our goal is not to provide an efficient or even effective implementation given that our algorithm is only decidable under assumption of Schanuel's conjecture. Instead our goal is to illustrate that using state-of-the-art solvers, we can get effective solutions in many cases, with interesting phenomena being observed in other cases. 

\subsection{Experimental Details}
Below, we have listed some example MDPs and MCs on which our algorithm has been tested. We have also compared the output of our algorithm with the ``best/optimal'' answer we could find for \initunfair \ (obtained by grid-search over all possible randomized policies, with the runs simulated for 1000 steps) and - wherever possible - lower bounds on the optimal answer possible under the restriction to convex sets as the invariant. Wherever the size/nature of the invariant set has been mentioned, we have aimed to mention the smallest invariant set under which the mentioned answer is given. This is done by starting with a template size of 10, and working one's way down to an invariant of size 1. Where it is mentioned that an empty invariant set suffices, that would imply that for a template of size 1, a trivial invariant such as $0 \ge 0$ is given.

Also, in some cases below instead of reporting the entropy we report the exponential of entropy $e^{H(\mu)}$ for convenience. Further, instead of saying $K=0,1,$ etc, we sometimes say entropy after step $K$, essentially meaning the same.

For examples taken from ~\cite{akshay2023mdps}, the structure and initialization of the MDPs/MCs have been taken from the Artefact Evaluation of the paper at \url{https://zenodo.org/records/7922231}.

\subsection{Benchmarks Details and Results}
\label{app:benchmarks}
We start by describing eight MDP examples, including the seven listed in the
main text.  In the numerical results reported below, all equalities should be
interpreted as approximate (with the last digit uncertain).

\subsubsection{MDP examples}
\begin{enumerate}
\renewcommand{\labelenumi}{\arabic{enumi}.}
\item MDP $M_1$ with 3 states, taken from running example of ~\cite{akshay2023mdps})
\begin{center}
\runningexamplefig

\end{center}
At the start, in $\mu_0$, state A has probability $\frac{1}{2}$, state B has probability $\frac{1}{3}$, state C has probability $\frac{1}{6}$.\\
In this example, it is optimal to choose action $a_1$ always. As the initial probability mass at $A$ exceeds $\frac{1}{e}$, any additional probability mass will only decrease its contribution to the entropy.
\begin{itemize}
\item $K=0$ Starting from $\mu_0$: With an invariant set of size 2, Gurobi chooses
  action $a_1$ at $A$ and outputs an answer of $e^{H(\mu)}=2.78718$. It is
  worthwhile to note that if one simulates a run of this MDP with this strategy,
  the maximum value of $e^{H(\mu)}$ turns out to be less than 2.7495. However, we
  find numerically that for a particular convex combination of the first two
  distributions (i.e., \(\mu_0\) and \(M^{\pi}(\mu_0, 1)\)) , the exponential of
  entropy is at least 2.78717, giving a lower bound on the optimal possible
  output of our algorithm.
\item $K=1$ Starting from \(M(\mu_0, 1)\): Here, the optimal value of
  $e^{H(\mu)}$ is approximately 1.97229, and Gurobi finds this with an invariant
  set of size two. \item $K=2$ Starting from $M(\mu_0, 2)$: Optimal value of $e^{H(\mu)}$ is
  approximately \(1.66819\), and this is found by Gurobi with an invariant set of
  size two. \end{itemize}
\item MDP $M_2$ with 4 states:
\begin{center}
\begin{tikzpicture}[scale=0.15]
\tikzstyle{every node}+=[inner sep=0pt]
\draw [black] (31.5,-15.9) circle (3);
\draw (31.5,-15.9) node {$3$};
\draw [black] (20,-27.8) circle (3);
\draw (20,-27.8) node {$1$};
\draw [black] (42.1,-27.8) circle (3);
\draw (42.1,-27.8) node {$2$};
\draw [black] (31.5,-40.1) circle (3);
\draw (31.5,-40.1) node {$0$};
\draw [black] (29.45,-37.91) -- (22.05,-29.99);
\fill [black] (22.05,-29.99) -- (22.23,-30.92) -- (22.96,-30.23);
\draw (21.28,-35.48) node [right] {$a,1$};
\draw [black] (33.46,-37.83) -- (40.14,-30.07);
\fill [black] (40.14,-30.07) -- (39.24,-30.35) -- (40,-31);
\draw (37.35,-35.4) node [right] {$b,1$};
\draw [black] (18.778,-25.077) arc (-166.00898:-282.03238:8.448);
\fill [black] (28.74,-14.77) -- (28.06,-14.12) -- (27.85,-15.09);
\draw (20.37,-15.69) node [left] {$4/5$};
\draw [black] (22.809,-26.75) arc (107.38534:72.61466:27.582);
\fill [black] (39.29,-26.75) -- (38.68,-26.03) -- (38.38,-26.99);
\draw (31.05,-24.99) node [above] {$1/5$};
\draw [black] (34.273,-14.802) arc (100.95769:-17.57119:8.078);
\fill [black] (34.27,-14.8) -- (35.15,-15.14) -- (34.96,-14.16);
\draw (42.38,-15.9) node [right] {$1/2$};
\draw [black] (39.306,-28.888) arc (-71.94154:-108.05846:26.634);
\fill [black] (22.79,-28.89) -- (23.4,-29.61) -- (23.71,-28.66);
\draw (31.05,-30.7) node [below] {$1/2$};
\draw [black] (29.42,-18.06) -- (22.08,-25.64);
\fill [black] (22.08,-25.64) -- (23,-25.41) -- (22.28,-24.72);
\draw (25.22,-20.38) node [left] {$1/3$};
\draw [black] (33.5,-18.14) -- (40.1,-25.56);
\fill [black] (40.1,-25.56) -- (39.95,-24.63) -- (39.2,-25.3);
\draw (37.34,-20.4) node [right] {$2/3$};
\end{tikzpicture}
\end{center}
The initial distribution $\mu_0$ is (0.2, 0.4, 0, 0.4).
\begin{itemize}
\item $K=0$ Starting from $\mu_0$: The answer given by our Gurobi based solver (in
  a specific run) with an invariant set of size two is $e^{H(\mu)}=3.7568$, with
  action \(a\) being chosen with probability about 0.4. With a brute force
  computation, we can check that the optimal memoryless policy involves
  selecting action \(a\) with probability about 0.788 (this number is slightly
  rounded off, see the below discussion for how to get the analytic value of the
  optimum), ans leads to a maximum value of $e^{H(\mu)}$ of at most $2.9776$.
\item $K=1$ Starting from $M(\mu_0, 1)$: our Gurobi based tool can be made (by
  choosing appropriate initial conditions) to give an answer of
  $e^{H(\mu)}=2.984322$ with an invariant of size five, with the probability of
  choosing action $a$ being \(0.57\). This is an example where the initial starting
  point for Gurobi matters. Note that under this particular probability,
  $M(\mu_0, 4)$ is a linear combination of $M(\mu_0, 1), M(\mu_0, 2), M(\mu_0, 3)$.
\begin{align*}
\mu^{(1)} &= \left(
0,
\frac{371}{1500},
\frac{649}{1500},
\frac{480}{1500}
\right)
\\[1em]
\mu^{(2)} &= \left(
0,
\frac{323}{1000},
\frac{1971}{7500},
\frac{2071}{5000}
\right)
\\[1em]
\mu^{(3)} &= \left(
0,
\frac{2021}{7500},
\frac{5111}{15000},
\frac{1949}{5000}
\right)
\\[1em]
\mu^{(4)} &= \left(
0,
\frac{3003}{10000},
\frac{3922}{12500},
\frac{19297}{50000}
\right)
\end{align*}
$\mu^{(4)}=\frac{3}{10}\mu^{(1)}+\frac{7}{10}\mu^{(2)}$. It then follows by induction
that the convex hull of the the first three points amongst these four, which can
be defined by four hyperplanes, serves as a valid invariant set: \(M(\mu_0, t)\)
lies in this set for all \(t \geq 1\).   

As in the \(K=0\) case above, a brute force computation shows that the strategy
of choosing action $a$ with probability 0.788 gives a maximum value of at most
$e^{H(\mu_0)}=2.9776$.
\end{itemize}
    \item A five-state MDP $M_3$ with two bottom strongly connected component
      and a transient state.  This is the same example as the one used in
      \cref{prop:no-state-yes-dist} above.
\begin{center}
\begin{tikzpicture}[scale=0.1]
    \tikzstyle{every node}+=[inner sep=0pt]
    \draw [black] (34.5,-14.9) circle (3);
    \draw [black] (24.2,-31.5) circle (3);
    \draw [black] (11.9,-31.5) circle (3);
    \draw [black] (44.3,-31.5) circle (3);
    \draw [black] (56.7,-31.5) circle (3);
    \draw [black] (32.92,-17.45) -- (25.78,-28.95);
    \fill [black] (25.78,-28.95) -- (26.63,-28.53) -- (25.78,-28.01);
    \draw (28.72,-21.91) node [left] {$a,0.7$};
    \draw [black] (13.055,-28.763) arc (143.32255:36.67745:6.228);
    \fill [black] (13.06,-28.76) -- (13.93,-28.42) -- (13.13,-27.82);
    \draw (18.05,-25.75) node [above] {$1$};
    \draw [black] (23.089,-34.255) arc (-35.79607:-144.20393:6.212);
    \fill [black] (23.09,-34.26) -- (22.22,-34.61) -- (23.03,-35.2);
    \draw (18.05,-37.33) node [below] {$1$};
    \draw [black] (45.391,-28.736) arc (144.72676:35.27324:6.258);
    \fill [black] (55.61,-28.74) -- (55.56,-27.79) -- (54.74,-28.37);
    \draw (50.5,-25.59) node [above] {$0.5$};
    \draw [black] (59.38,-30.177) arc (144:-144:2.25);
    \draw (63.95,-31.5) node [right] {$0.5$};
    \fill [black] (59.38,-32.82) -- (59.73,-33.7) -- (60.32,-32.89);
    \draw [black] (55.21,-34.073) arc (-43.35523:-136.64477:6.477);
    \fill [black] (45.79,-34.07) -- (45.98,-35) -- (46.7,-34.31);
    \draw (50.5,-36.6) node [below] {$0.5$};
    \draw [black] (41.477,-32.48) arc (316.87498:28.87498:2.25);
    \draw (36.93,-30.53) node [left] {$0.5$};
    \fill [black] (41.81,-29.86) -- (41.56,-28.94) -- (40.88,-29.67);
    \draw [black] (36.03,-17.48) -- (42.77,-28.92);
    \fill [black] (42.77,-28.92) -- (42.8,-27.97) -- (41.94,-28.48);
    \draw (38.75,-24.45) node [left] {$b,0.05$};
    \draw [black] (37.466,-14.473) arc (93.45714:12.96842:18.172);
    \fill [black] (56.27,-28.53) -- (56.58,-27.64) -- (55.6,-27.87);
    \draw (52.45,-17.56) node [above] {$b,0.05$};
    \draw [black] (31.689,-13.886) arc (277.90143:-10.09857:2.25);
    \draw (27.3,-9.55) node [above] {$b,0.9$};
    \fill [black] (33.59,-12.05) -- (33.98,-11.19) -- (32.99,-11.33);
    \draw [black] (35.359,-12.038) arc (191.02106:-96.97894:2.25);
    \draw (41.56,-9.48) node [above] {$a,0.3$};
    \fill [black] (37.29,-13.84) -- (38.18,-14.18) -- (37.98,-13.2);
\end{tikzpicture}
\label{fig:counterexample}
\end{center}
At the start, i.e., in $\mu_0$, the entire probability mass is in the top state.
We can parameterize all memoryless strategies in terms of a single parameter
\(x \in [0, 1]\), and can then plot the entropies at times
\(t \in \inb{1, 2, 3}\) as a function of \(x\), as described in the proof of
\cref{prop:no-state-yes-dist}. As noted there, the maximum of these three
functions is always above 0.77 (from the same computation, one can also find
that the best achievable value of the exponential of the entropy under
memoryless strategies is about \(2.164\) ), whereas with a policy alternating
between choosing \(a\) and \(b\), we can achieve a maximum entropy below
\(0.73\). Our Gurobi based solver reports an answer for entropy of about
\(\log 3\), and this is expected, since under the restriction to convex
invariants, the minimum possible maximum entropy is no less than $\ln 3$. To see
this, we note that in the convex hull of the initial distribution and
\(M(\mu_0, n)\), \(M(\mu_0, n+1)\) for large \(n\), there lie distributions with
entropy arbitrary close to $\ln{3}$.
\item MDP $M_4$ with 5 states and 2 actions : 
\begin{center}
\begin{tikzpicture}[scale=0.1]
\tikzstyle{every node}+=[inner sep=0pt]
\draw [black] (18.2,-25.8) circle (3);
\draw (18.2,-25.8) node {A};

\draw [black] (27.7,-25.8) circle (3);
\draw (27.7,-25.8) node {B};

\draw [black] (37.7,-25.8) circle (3);
\draw (37.7,-25.8) node {C};

\draw [black] (47.7,-25.8) circle (3);
\draw (47.7,-25.8) node {D};

\draw [black] (37.7,-35.7) circle (3);
\draw (37.7,-35.7) node {E};

\draw [black] (21.2,-25.8) -- (24.7,-25.8);
\fill [black] (24.7,-25.8) -- (23.9,-25.3) -- (23.9,-26.3);

\draw [black] (30.7,-25.8) -- (34.7,-25.8);
\fill [black] (34.7,-25.8) -- (33.9,-25.3) -- (33.9,-26.3);

\draw [black] (40.7,-25.8) -- (44.7,-25.8);
\fill [black] (44.7,-25.8) -- (43.9,-25.3) -- (43.9,-26.3);

\draw [black] (37.7,-32.7) -- (37.7,-28.8);
\fill [black] (37.7,-28.8) -- (37.2,-29.6) -- (38.2,-29.6);
\draw (38.2,-30.75) node [right] {$a,1$};

\draw [black] (39.023,-38.38) arc (54:-234:2.25);
\draw (37.7,-42.95) node [below] {$b,1$};
\fill [black] (36.38,-38.38) -- (35.5,-38.73) -- (36.31,-39.32);

\draw [black] (50.38,-24.477) arc (144:-144:2.25);
\fill [black] (50.38,-27.12) -- (50.73,-28) -- (51.32,-27.19);
\end{tikzpicture}
\end{center}
At the start, with $\mu_0$ as state A has probability $\frac{9}{10}$ and state E has probability $\frac{1}{10}$.
\begin{itemize}
\item $K=2$ Starting from \(M(\mu_0, 2)\): Our Gurobi based solver achieves a
  value of $e^{H(\mu)}$ of roughly \(2.0\), even with an invariant of size one..
  The optimal memoryless policy, in comparison, yields an answer of
  $e^{H(\mu)} \approxeq 1.35506$, with the probability of choosing action \(b\) being
  roughly 0.252241: this can be obtained by parameterizing such strategies as in
  the proof of \cref{prop:no-state-yes-dist}, which suggests minimizing the
  two-step entropy given by \\ \begin{equation*} -\Bigg( \left(0.9 + 0.1 x (1 -
      x) \right) \ln\left(0.9 + 0.1 x (1 - x)\right) + 0.1 (1 - x) \ln\left(0.1 (1 - x)\right) + 0.1 x^2 \ln\left(0.1 x^2\right) \Bigg),
\end{equation*}
  and then verifying the resulting strategy.
 
  This is again an example of an MDP where a randomized policy is better than
  any deterministic. Either deterministic policy would give a maximum value of
  roughly $e^{H(\mu)}=1.38414$ over the run of the embedded Markov chain.
  Interestingly, the optimal policy under the restriction to convex sets gives a
  maximum entropy of $\ln{2}$, with the deterministic policy that always chooses
  action \(a\) at state \(E\). This is because with this strategy, there exists
  a point in the convex hull of \(M^{\pi}(\mu_0, 2)\) and
  \(M^{\pi}(\mu_0, 3)\) whose entropy is $\ln(2)$, and for a strategy of any other
  probability of selecting \(a\), there is a point of higher entropy in this
  convex hull.
\end{itemize}
\item MDP $M_5$ with 3 states, with warm-up parameter \(K=0\):
\begin{center}
\begin{tikzpicture}[scale=0.16]
\tikzstyle{every node}+=[inner sep=0pt]
\draw [black] (32.6,-18.6) circle (3);
\draw (32.6,-18.6) node {$A$};
\draw [black] (21.1,-33.8) circle (3);
\draw (21.1,-33.8) node {$B$};
\draw [black] (43.6,-33.8) circle (3);
\draw (43.6,-33.8) node {$C$};
\draw [black] (31.277,-15.92) arc (234:-54:2.25);
\draw (32.6,-11.35) node [above] {$a,1/2$};
\fill [black] (33.92,-15.92) -- (34.8,-15.57) -- (33.99,-14.98);
\draw [black] (22.188,-31.005) arc (155.83642:129.94282:29.66);
\fill [black] (22.19,-31.01) -- (22.97,-30.48) -- (22.06,-30.07);
\draw (26.30,-23.85) node [left] {{\small $a,1/4$}};
\draw [black] (41.319,-31.853) arc (-133.03184:-155.18283:33.663);
\fill [black] (41.32,-31.85) -- (41.08,-30.94) -- (40.39,-31.67);
\draw (37.43,-29.36) node [left] {{\small $a,1/4$}};
\draw [black] (34.633,-16.41) arc (164.85446:-123.14554:2.25);
\draw (39.66,-15.36) node [right] {$b,1/4$};
\fill [black] (35.57,-18.88) -- (36.22,-19.57) -- (36.48,-18.61);
\draw [black] (34.779,-20.661) arc (44.57648:27.20885:42.615);
\fill [black] (42.32,-31.09) -- (42.4,-30.15) -- (41.51,-30.6);
\draw (38.73,-24.21) node [right] {$b,1/4$};
\draw [black] (31.169,-21.236) arc (-30.093:-44.12777:53.731);
\fill [black] (23.25,-31.71) -- (24.16,-31.48) -- (23.45,-30.78);
\draw (28.1,-27.11) node [right] {{\small $b,1/2$}};
\draw [black] (41.385,-35.816) arc (-53.36933:-126.63067:15.143);
\fill [black] (41.38,-35.82) -- (40.44,-35.89) -- (41.04,-36.69);
\draw (32.35,-39.31) node [below] {$1/3$};
\draw [black] (19.514,-36.333) arc (-4.31364:-292.31364:2.25);
\draw (14.72,-38.63) node [left] {$1/3$};
\fill [black] (18.13,-34.08) -- (17.36,-33.52) -- (17.29,-34.52);
\draw [black] (19.516,-31.266) arc (-156.80718:-277.41359:9.745);
\fill [black] (29.73,-17.76) -- (29,-17.17) -- (28.87,-18.16);
\draw (20.13,-20.15) node [left] {$1/3$};
\draw [black] (35.456,-17.724) arc (98.01666:-26.23133:9.521);
\fill [black] (35.46,-17.72) -- (36.32,-18.11) -- (36.18,-17.12);
\draw (45.08,-20.19) node [right] {$2/5$};
\draw [black] (46.557,-34.228) arc (109.49148:-178.50852:2.25);
\draw (49.76,-38.98) node [right] {$2/5$};
\fill [black] (45.06,-36.41) -- (44.85,-37.33) -- (45.8,-37);
\draw [black] (40.6,-33.8) -- (24.1,-33.8);
\fill [black] (24.1,-33.8) -- (24.9,-34.3) -- (24.9,-33.3);
\draw (32.35,-34.3) node [below] {$1/5$};
\end{tikzpicture}
\end{center}
Initially, all the probability mass is in A, i.e., $\mu_0=(1,0,0)$  .
With an invariant of size one, our Gurobi based tool reports an answer of
$e^{H(\mu)}=2.9543$, choosing action \(a\) with probability one. By a brute force
computation we check that this is indeed the optimal memoryless policy for this
problem, but it achieves a somewhat better answer of $e^{H(\mu)}=2.93492$, again
highlighting that restriction to a convex invariant can be
over-cautious. 

\item Split (MDP with 4 states and 2 disconnected components taken from     ~\cite{akshay2023mdps}): 
\begin{center}
\begin{tikzpicture}[scale=0.15]
\tikzstyle{every node}+=[inner sep=0pt]
\draw [black] (22.3,-23.1) circle (3);
\draw (22.3,-23.1) node {$A$};
\draw [black] (22.3,-35.9) circle (3);
\draw (22.3,-35.9) node {$B$};
\draw [black] (36.6,-23.1) circle (3);
\draw (36.6,-23.1) node {$C$};
\draw [black] (36.3,-35.8) circle (3);
\draw (36.3,-35.8) node {$D$};
\draw [black] (24.843,-24.645) arc (45.98655:-45.98655:6.751);
\fill [black] (24.84,-34.36) -- (25.77,-34.16) -- (25.07,-33.44);
\draw (27.4,-29.5) node [right] {$b,1$};
\draw [black] (20.977,-20.42) arc (234:-54:2.25);
\draw (22.3,-15.85) node [above] {$a,0.9$};
\fill [black] (23.62,-20.42) -- (24.5,-20.07) -- (23.69,-19.48);
\draw [black] (22.3,-26.1) -- (22.3,-32.9);
\fill [black] (22.3,-32.9) -- (22.8,-32.1) -- (21.8,-32.1);
\draw (21.8,-29.5) node [left] {$a,0.1$};
\draw [black] (37.623,-38.48) arc (54:-234:2.25);
\draw (36.3,-43.05) node [below] {$1$};
\fill [black] (34.98,-38.48) -- (34.1,-38.83) -- (34.91,-39.42);
\draw [black] (23.623,-38.58) arc (54:-234:2.25);
\draw (22.3,-43.15) node [below] {$1$};
\fill [black] (20.98,-38.58) -- (20.1,-38.93) -- (20.91,-39.52);
\draw [black] (36.53,-26.1) -- (36.37,-32.8);
\fill [black] (36.37,-32.8) -- (36.89,-32.01) -- (35.89,-31.99);
\draw (35.91,-29.44) node [left] {$0.5$};
\draw [black] (35.277,-20.42) arc (234:-54:2.25);
\draw (36.6,-15.85) node [above] {$0.5$};
\fill [black] (37.92,-20.42) -- (38.8,-20.07) -- (37.99,-19.48);
\end{tikzpicture}
\end{center}
At the start, i.e., at $\mu_0=0$, state A has probability $\frac{1}{3}$ and state C has probability $\frac{2}{3}$.
We consider the case where the warm-up parameter satisfies \(K=0\). Our Gurobi
based solved reports an answer of about $e^{H(\mu)}=3.77976$, with an invariant
set of size one ($A+B \le \frac{1}{3}$), and the strategy of always choosing
action \(b\). On simulating the MDP with this strategy, we find that the maximum
entropy attained is at most $\log 3$. However, under the restriction to convex
invariant sets, the optimal answer for this strategy is indeed at least
$e^{H(\mu)}=3.77976$, as $\mu_0=(\frac{1}{3}, 0, \frac{2}{3}, 0)$ and the
probability distributions \(M^{\pi}(\mu_0, n)\) converge to
$(0, \frac{1}{3}, 0, \frac{2}{3})$ as \(n \rightarrow \infty\), so that the entropy of the
average of $\mu_0$ and $M^{\pi}(\mu_0, n)$ for high enough \(n\) comes arbitrarily
close to $\ln(3.77976)$ from below.
\item Gridworld 1 (MDP with 8 states): This is a grid of size (2,4) with the
  point (0,1) being unreachable, as shown in the figure with (0,0) being
  top-left and (1,3) being bottom-right. Note that all transition probabilities
  that are not labelled on edges are 1. Also there is a non-deterministic action
  at state (0,2), where one can go down or right, and a self-loop at (1,3)
  (thus, this state is an absorbing state). The initial distribution $\mu_0$
  assigns all probability mass and the top-left state (0,0).
\begin{center}
\begin{tikzpicture}[scale=0.1]
\tikzstyle{every node}+=[inner sep=0pt]
\draw [black] (18.9,-24.2) circle (3);
\draw [black] (28.2,-24.2) circle (3);
\draw [black] (38.3,-24.2) circle (3);
\draw [black] (47.7,-24.2) circle (3);
\draw [black] (18.9,-33.9) circle (3);
\draw [black] (28.2,-33.9) circle (3);
\draw [black] (38.3,-33.9) circle (3);
\draw [black] (47.7,-33.9) circle (3);
\draw [black] (41.3,-24.2) -- (44.7,-24.2);
\fill [black] (44.7,-24.2) -- (43.9,-23.7) -- (43.9,-24.7);
\draw [black] (38.3,-27.2) -- (38.3,-30.9);
\fill [black] (38.3,-30.9) -- (38.8,-30.1) -- (37.8,-30.1);
\draw [black] (47.7,-27.2) -- (47.7,-30.9);
\fill [black] (47.7,-30.9) -- (48.2,-30.1) -- (47.2,-30.1);
\draw [black] (41.3,-33.9) -- (44.7,-33.9);
\fill [black] (44.7,-33.9) -- (43.9,-33.4) -- (43.9,-34.4);
\draw [black] (18.9,-27.2) -- (18.9,-30.9);
\fill [black] (18.9,-30.9) -- (19.4,-30.1) -- (18.4,-30.1);
\draw [black] (21.9,-33.9) -- (25.2,-33.9);
\fill [black] (25.2,-33.9) -- (24.4,-33.4) -- (24.4,-34.4);
\draw [black] (31.2,-33.9) -- (35.3,-33.9);
\fill [black] (35.3,-33.9) -- (34.5,-33.4) -- (34.5,-34.4);
\end{tikzpicture}
\end{center}
There is no other state (except (1, 3)) with no possible transition to another state. At any other state, one of up to two actions - Down or Right - can be chosen (these are not depicted in the figure, with the figure only showing the structure of MDP), provided the action does not force one off the board or into an unreachable state. Note that in our implementation we simply remove unreachable states, e.g., (0,1) above. 

Note that there is no randomness in this example, and the entropy of every
resulting distribution is \(0\). However, because we are restricted to finding
the largest entropy possible for distributions satisfying a convex invariant,
the smallest such invariant will be the convex hull of all these resulting
distributions. Our solver should report as the optimal value of \(e^{H(u)}\) the
number of vertices along the shortest path from the source to the sink. Indeed,
with the warm-up parameter set to \(0\), and an invariant set of size two, our
Gurobi based solver estimates an optimal answer of approximately 5
(corresponding to a shortest path of length four containing five vertices).

\item Gridworld 2 (MDP with 12 states): We have a grid of size (3,4) with the
  unreachable states as (1, 1), (2, 1) and (1, 3) (with \((0, 0)\) being top
  left and \((2, 3)\) the bottom right). At any other state, one of up to two
  actions - Down or Right - can be chosen, provided the action does not force
  one off the board or into an unreachable state. All transitions that are drawn
  in the picture are assumed to have probability 1. Further, the state (2, 3) is
  absorbing, so has a self loop with prob 1. From any reachable state except (2,
  3) which does not have transitions marked in the figure, in the actual model
  there is a transition from that state to all other reachable states with
  uniform probability. The initial distribution $\mu_0$ is top-left state (0,0)
  having probability
  1. \begin{center}
\begin{tikzpicture}[scale=0.1]
\tikzstyle{every node}+=[inner sep=0pt]
\draw [black] (18.9,-24.2) circle (3);
\draw [black] (28.2,-24.2) circle (3);
\draw [black] (38.3,-24.2) circle (3);
\draw [black] (47.7,-24.2) circle (3);
\draw [black] (18.9,-33.9) circle (3);
\draw [black] (28.2,-33.9) circle (3);
\draw [black] (38.3,-33.9) circle (3);
\draw [black] (47.7,-33.9) circle (3);
\draw [black] (18.9,-44.2) circle (3);
\draw [black] (28.2,-44.2) circle (3);
\draw [black] (38.3,-44.2) circle (3);
\draw [black] (47.7,-44.2) circle (3);
\draw [black] (41.3,-24.2) -- (44.7,-24.2);
\fill [black] (44.7,-24.2) -- (43.9,-23.7) -- (43.9,-24.7);
\draw [black] (38.3,-27.2) -- (38.3,-30.9);
\fill [black] (38.3,-30.9) -- (38.8,-30.1) -- (37.8,-30.1);
\draw [black] (18.9,-27.2) -- (18.9,-30.9);
\fill [black] (18.9,-30.9) -- (19.4,-30.1) -- (18.4,-30.1);
\draw [black] (21.9,-24.2) -- (25.2,-24.2);
\fill [black] (25.2,-24.2) -- (24.4,-23.7) -- (24.4,-24.7);
\draw [black] (31.2,-24.2) -- (35.3,-24.2);
\fill [black] (35.3,-24.2) -- (34.5,-23.7) -- (34.5,-24.7);
\draw [black] (18.9,-36.9) -- (18.9,-41.2);
\fill [black] (18.9,-41.2) -- (19.4,-40.4) -- (18.4,-40.4);
\draw [black] (38.3,-36.9) -- (38.3,-41.2);
\fill [black] (38.3,-41.2) -- (38.8,-40.4) -- (37.8,-40.4);
\draw [black] (41.3,-44.2) -- (44.7,-44.2);
\fill [black] (44.7,-44.2) -- (43.9,-43.7) -- (43.9,-44.7);
\end{tikzpicture}
\end{center}

Note again that while there is some randomness in this example, it is similar to
the previous example in that there is a strategy under which the entropy of
every resulting distribution is \(0\). However, again, because we are restricted
to finding the largest entropy possible for distributions satisfying a convex
invariant, the smallest such invariant will be the convex hull of all these
resulting distributions. Our solver should thus report as the optimal value of
\(e^{H(u)}\) the number of vertices along the shortest path from the source to
the sink, , and should find the strategy that induces this shortest path.
Indeed, with the warm-up parameter set to \(0\), and an invariant set of size
two, our Gurobi based solver estimates an optimal answer of approximately 6
(corresponding to a shortest path of length five containing six vertices), and
find the strategy that induces this shortest path (Right at \(0, 0\), and Down
at \((0, 2)\)), using an invariant set of size one.

\end{enumerate}

\subsubsection{Markov chain examples}

Below we describe our Markov chain examples.
\begin{enumerate}

    \item Markov chain $MC1$ with two states
\begin{center}
\begin{tikzpicture}[scale=0.2]
\tikzstyle{every node}+=[inner sep=0pt]

\draw [black] (47.3,-26) circle (3);
\draw (47.3,-26) node {B};

\draw [black] (31.8,-26) circle (3);
\draw (31.8,-26) node {A};

\draw [black] (49.98,-24.677) arc (144:-144:2.25);
\fill [black] (49.98,-27.32) -- (50.33,-28.2) -- (50.92,-27.39);
\draw (54.55,-26) node [right] {1};

\draw [black] (34.8,-26) -- (44.3,-26);
\fill [black] (44.3,-26) -- (43.5,-25.5) -- (43.5,-26.5);
\draw (39.55,-26) node [above] {0.5};

\draw [black] (29.12,-27.323) arc (324:36:2.25);
\fill [black] (29.12,-24.68) -- (28.77,-23.8) -- (28.18,-24.61);
\draw (24.55,-26) node [left] {0.5};

\end{tikzpicture}
\end{center}
At the start, i.e., in $\mu_0$, state A has probability $\frac{2}{3}$ and state B has probability $\frac{1}{3}$.
    \begin{itemize}
    \item $K=0$ Starting from $\mu_0$: Our Gurobi based solver reports an answer
      of $e^{H(\mu)}=2.000$, with an invariant set of size one. By simulating a
      run of the MC, the maximum value of $e^H(\mu)$ over the run is 1.88988.
      However, since the distribution \((1/2, 1/2)\) is contained in the convex
      hull of \(\mu_0\) and \(M(\mu_0, 1)\), the answer reported by our solver is
      the best possible for any convex invariant based method.
    \item $K=1$ Starting from \(M(\mu_0, 1)\): The solver reports the
      approximately correct optimal value of $e^{H(\mu)}=1.88988$, using an
      invariant set of size one.
    \item $K=2$ Starting from \(M(\mu_0, 2)\): The solver reports the
      approximately correct optimal value of $e^{H(\mu)}=1.88988$, using an
      invariant set of size one. \end{itemize}
\item Markov chain $MC2$ with two states:
\begin{center}
\begin{tikzpicture}[
  state/.style={circle, draw, minimum size=1.2cm, inner sep=2pt},
  arrow/.style={->, thick},
  loop right/.style={->, thick, looseness=8, out=30, in=-30},
  loop left/.style={->, thick, looseness=8, out=150, in=-150}
]

\node[state] (A) at (0,0) {A};
\node[state] (B) at (3,0) {B};

\draw[arrow] (A) to[bend left=30] node[above] {1} (B);
\draw[arrow] (B) to[bend left=30] node[below] {0.5} (A);
\draw[arrow, loop right] (B) to node[right] {0.5} (B);
\end{tikzpicture}
\end{center}
At the start, i.e., in $\mu_0$, both A and B have $\frac{1}{2}$ of the probability each. \begin{itemize}
    \item $K=1,2$, starting from \(M(\mu_0, 1)\) or \(M(\mu_0, 2)\): Our Gurobi based
      solver outputs the correct optimal value of $e^{H(\mu)}=1.93782$, with an
      invariant set of size two. \item \(K = 3\), starting from \(M(\mu_0, 3)\): Our Gurobi based
      solver outputs the correct optimal value of
      $e^{H(\mu)}=1.903114$ with an invariant set of size two.
\end{itemize}
\item Markov chain MC3, an ergodic Markov chain with three states: 
\begin{center}
\begin{tikzpicture}[scale=0.18]
\tikzstyle{every node}+=[inner sep=0pt]
\draw [black] (23.8,-24.4) circle (3);
\draw (23.8,-24.4) node {$A$};
\draw [black] (41.4,-24.4) circle (3);
\draw (41.4,-24.4) node {$B$};
\draw [black] (32.3,-36.7) circle (3);
\draw (32.3,-36.7) node {$C$};
\draw [black] (21.031,-23.276) arc (275.63354:-12.36646:2.25);
\draw (17.55,-18.87) node [above] {$0.5$};
\fill [black] (23.01,-21.52) -- (23.43,-20.67) -- (22.43,-20.77);
\draw [black] (26.8,-24.4) -- (38.4,-24.4);
\fill [black] (38.4,-24.4) -- (37.6,-23.9) -- (37.6,-24.9);
\draw (32.6,-24.9) node [below] {$0.25$};
\draw [black] (29.41,-35.939) arc (-113.47737:-177.22927:9.859);
\fill [black] (29.41,-35.94) -- (28.88,-35.16) -- (28.48,-36.08);
\draw (24.63,-33.86) node [left] {$0.25$};
\draw [black] (33.623,-39.38) arc (54:-234:2.25);
\draw (32.3,-43.95) node [below] {$0.5$};
\fill [black] (30.98,-39.38) -- (30.1,-39.73) -- (30.91,-40.32);
\draw [black] (41.493,-27.388) arc (-6.35527:-66.63557:10.561);
\fill [black] (41.49,-27.39) -- (40.91,-28.13) -- (41.9,-28.24);
\draw (40.07,-33.89) node [right] {$0.25$};
\draw [black] (30.59,-34.23) -- (25.51,-26.87);
\fill [black] (25.51,-26.87) -- (25.55,-27.81) -- (26.37,-27.24);
\draw (28.65,-29.19) node [right] {$0.25$};
\draw [black] (39.62,-26.81) -- (34.08,-34.29);
\fill [black] (34.08,-34.29) -- (34.96,-33.94) -- (34.16,-33.35);
\draw (36.27,-29.16) node [left] {$0.25$};
\draw [black] (25.242,-21.784) arc (141.9438:38.0562:9.345);
\fill [black] (25.24,-21.78) -- (26.13,-21.46) -- (25.34,-20.85);
\draw (32.6,-17.7) node [above] {$0.25$};
\draw [black] (42.602,-21.664) arc (184.00853:-103.99147:2.25);
\draw (47.12,-18.62) node [right] {$0.5$};
\fill [black] (44.3,-23.69) -- (45.14,-24.13) -- (45.07,-23.14);
\end{tikzpicture}
\end{center}
Initially, the probability distribution $\mu_0$ is (1, 0, 0).
Starting from $\mu_0$ (\(K=0\)), the Gurobi based solver reports the optimal
answer of $e^{H(\mu)}=3$.

\item Insulin: Pharmokinetics benchmark taken originally
  from~\cite{DBLP:conf/qest/ChadhaKVAK11}, but as modified in~\cite[Example
  2]{DBLP:journals/jacm/AgrawalAGT15}, which we describe briefly. This Markov
  chain has five non-dummy nodes where three of the nodes correspond to the
  ``body compartments'' Plasma (Pl), Intersticial Fluid (IF), and ``Utilization
  and degradation'' (Ut). The node (Dr) models the drug being injected, and the
  node (Cl) models ``the drug being cleared from the body after degradation.''
  The transition matrix of the chain is (with the first row corresponding to
  state Dr,
  and the last to state Cl):
  \begin{equation*}
    \begin{bmatrix}
       \frac{4699}{5000} & \frac{1317}{50000} & \frac{641}{25000} & \frac{399}{50000} & \frac{3}{12500}\\
       0 & \frac{5181}{25000} & \frac{24149}{50000} & \frac{3703}{12500} & \frac{677}{50000} \\
       0 & \frac{15531}{100000} & \frac{42539}{100000} & \frac{3953}{10000} & \frac{3}{125} \\
       0 & \frac{1299}{50000} & \frac{5389}{50000} & \frac{38927}{50000} & \frac{877}{10000}\\
       0 & 0 & 0 & 0 & 1
    \end{bmatrix}
  \end{equation*}

  At the start, state Dr has probability $\frac{3}{10}$ and state Cl has
  probability $\frac{7}{10}$. For this problem, starting from $\mu_0$, the answer
  reported by our Gurobi based solver seems to depend quite significantly on the
  seed, with one particular run reporting the value $e^{H(\mu)}=2.72$, obtained
  with an invariant of size five (for this example, it also helps to run the
  solver without the pre-imposed bounds on the magnitudes of the template
  variables). The optimal answer, obtained by simulating the Markov chain, is
  $e^{H(\mu)}=2.3717679$.

\item Pagerank (A 5-state Markov chain taken
  from~\cite{DBLP:journals/jacm/AgrawalAGT15}[Example 1, Fig 3.]). This chains
  has five states, with the transition matrix given by
  \begin{equation*}
    \begin{bmatrix}
      \frac {1} {80} &\frac {19} {60} & \frac {3} {40} &\frac {19} {60} &\frac {67} {240} \\
      \frac {1} {80} &\frac {1} {20} &\frac {41} {120} &\frac {19} {60} &\frac {67} {240} \\
      \frac {1} {16} &\frac {1} {4} &\frac {3} {8} &\frac {1} {4} &\frac {1} {16} \\
      \frac {1} {80} &\frac {1} {20} &\frac {7} {8} &\frac {1} {20}  &\frac {1} {80} \\
      \frac {33} {80} &\frac {9} {20} &\frac {3} {40} &\frac {1} {20} &\frac {1} {80}
    \end{bmatrix},
  \end{equation*}
  with the rows and column corresponding to states A, B, C, D, and E in
  alphabetical order. At the start, the probability distribution over the states
  A, B, C, D, and E is
  $(\frac{1}{16}, \frac{4}{16}, \frac{6}{16}, \frac{4}{16}, \frac{1}{16})$.
Starting from $\mu_0$ ($K=0$), the answer given by our Gurobi based solver with
  an invariant set of size three is about $e^{H(\mu)}=4.65$ (for this problem, it
  helps to run the solver without any pre-imposed bounds on the template
  variables, and the answers reported by the solver seems to depend
  significantly on the initial random seed provided to it). The optimal value,
  found by simulation, is $e^{H(\mu)}=4.231614$.
\end{enumerate}

\subsection{Summary of Results}
\label{app:summary}
We summarize the above results in Table~\ref{tab:full-results-log}. As can be seen,
in most benchmarks the results given by our approach are not far from the
optimal results (empirically determined using simulations in most
cases). Further experiments show that our approach works on even larger
benchmarks, but as there seems no way to compare those results against optimal
results in those cases (due to difficulty of finding the optimal answer in large
systems) in, we do not describe them here.

Thus, in summary, the prototype implementation of our method works on a variety
of MDPs and MCs with different features, producing results that are not far from
an empirically observed optimal.

\begin{table}[!ht]
\centering
\begin{tabular}{|l|c|c|c|c|c|c|}
\hline
\textbf{Name} & \textbf{\#States} & \textbf{K} & \textbf{exp(Answer Given)} & \textbf{exp(Optimal)} & $\textbf{EntAns}$ & $\textbf{EntOpt}$ \\
\hline
\textbf{Type} & \textbf{MC} &   &   &   &   &   \\
\hline
MDP $M_1$ & 3 & 0 & 2.782 & 2.749 & 1.023 & 1.011 \\
\hline
MDP $M_1$ & 3 & 1 & 1.972 & 1.972 & 0.679 & 0.679 \\
\hline
MDP $M_1$ & 3 & 2 & 1.668 & 1.668 & 0.512 & 0.512 \\
\hline
MDP $M_2$ & 4 & 0 & 3.757 & 2.978 & 1.324 & 1.091 \\
\hline
MDP $M_2$ & 4 & 1 & 2.984 & 2.978 & 1.093 & 1.091 \\
\hline
MDP $M_3$ & 5 & 0  & 3.000 & 2.164 & 1.099 & 0.772 \\
\hline
MDP $M_4$ & 5 & 2 & 2.000 & 1.355 & 0.693 & 0.304 \\
\hline
MDP $M5$ & 3 & 0 & 2.954 & 2.935 & 1.083 & 1.077 \\
\hline
Split & 4 & 0 & 3.780 & 3.000 & 1.330 & 1.099 \\
\hline
\textbf{Type} & \textbf{MC} &  &  &  &  &  \\
\hline
MC1 & 2 & 0 & 2.000 & 1.890 & 0.693 & 0.637 \\
\hline
MC1 & 2 & 1 & 1.890 & 1.890 & 0.637 & 0.637 \\
\hline
MC1 & 2 & 2 & 1.890 & 1.890 & 0.637 & 0.637 \\
\hline
MC2 & 2 & 1 & 1.938 & 1.938 & 0.662 & 0.662 \\
\hline
MC2  & 2 & 2 & 1.938 & 1.928 & 0.662 & 0.656 \\
\hline
MC2 & 2 & 3 & 1.903 & 1.903 & 0.643 & 0.643 \\
\hline
MC3 & 3 & 0 & 3.000 & 3.000 & 1.099 & 1.099 \\
\hline
Insulin & 5 & 0 & 2.720 & 2.372 & 1.001 & 0.864 \\
\hline
Pagerank & 5 & 0 & 4.650 & 4.232 & 1.537 & 1.443 \\
\hline

\end{tabular}

\caption{Comparison of answer reported by our prototype tool vs Optimal Answer
  across different examples. $\mathbf{K}$ denotes the warm-up parameter. Optimal
  answers are derived by simulation (MCs) or grid search over all MDP policies.
  We report entropy values as well as their exponential values. All values
  rounded to three places after decimal.}
\label{tab:full-results-log}
\end{table}

\end{document}